\pdfoutput=1

\documentclass[11pt]{article}

\PassOptionsToPackage{dvipsnames,usename}{xcolor}
\usepackage[preprint]{acl}
\usepackage{times}
\usepackage{latexsym}
\usepackage{xcolor}

\usepackage[T1]{fontenc}
\usepackage[utf8]{inputenc}

\usepackage{microtype}

\usepackage{inconsolata}
\usepackage{amsmath}
\usepackage{amssymb}
\usepackage{amsthm}
\usepackage{amsfonts}
\usepackage{bbm}
\usepackage{xspace}
\usepackage{booktabs}
\usepackage{multicol}
\usepackage{scalerel}
\usepackage{thmtools} 
\usepackage{thm-restate}
\usepackage{caption}
\usepackage{subcaption}
\usepackage[frozencache,cachedir=.]{minted}
\setminted[python]{breaklines, framesep=2mm, fontsize=\footnotesize, numbersep=5pt}

\usepackage{pifont}%
\newcommand{\cmark}{\ding{51}}%
\newtheorem{defin}{Definition}
\newtheorem{lemma}{Lemma}

\usepackage{cleveref}
\crefname{section}{\S}{\S\S}
\Crefname{section}{\S}{\S\S}
\crefname{table}{Tab.}{}
\crefname{figure}{Fig.}{Figs.}
\crefname{algorithm}{Alg.}{}
\crefname{equation}{Eq.}{}
\crefname{appendix}{App.}{}
\crefname{theorem}{Thm.}{}
\crefname{proposition}{Proposition}{}
\crefname{defin}{Defn.}{}
\crefname{cor}{Corollary}{}
\crefname{observation}{Observation}{}
\crefname{assumption}{Assumption}{}
\crefformat{section}{\S#2#1#3}
\crefformat{footnote}{#2\footnotemark[#1]#3}

\usepackage{tabularray}
\usepackage{cellspace}
\newcommand\cincludegraphics[2][]{\raisebox{-0.3\height}{\includegraphics[#1]{#2}}}
\usepackage{setspace}
\makeatletter
\DeclareRobustCommand*{\escapeus}[1]{%
    \begingroup\@activeus\scantokens{#1\endinput}\endgroup}
\begingroup\lccode`\~=`\_\relax
\lowercase{\endgroup\def\@activeus{\catcode`\_=\active \let~\_}}
\makeatother

\usepackage[scaled=0.85]{helvet}
\newcommand{\makesf}[1]{\textsf{{\escapeus{#1}}}}

\usepackage{tikz}
\usetikzlibrary{decorations.pathreplacing}

\DeclareMathOperator*{\expect}{\mathbb{E}}
\DeclareMathOperator*{\argmax}{\mathrm{argmax}}
\DeclareMathOperator*{\argmin}{\mathrm{argmin}}

\definecolor{custommidnightblue}{HTML}{377EB8}
\definecolor{customforestgreen}{HTML}{4DAF4A}
\definecolor{customburntorange}{HTML}{DF813A}
\definecolor{custompurple}{HTML}{CF351B}

\newcommand{\colourbase}{black}
\newcommand{\colourword}{custompurple}        %
\newcommand{\colourunnormed}{customburntorange} %
\newcommand{\colourlocal}{custommidnightblue}  %
\newcommand{\colourglobal}{customforestgreen}  %

\newcommand{\mymacro}[2]{\newcommand{#1}{{\color{\colourbase}#2}}}
\newcommand{\myword}[2]{\newcommand{#1}{{\color{\colourword}#2}}}
\newcommand{\myunnormed}[2]{\newcommand{#1}{{\color{\colourunnormed}#2}}}
\newcommand{\mylocal}[2]{\newcommand{#1}{{\color{\colourlocal}#2}}}
\newcommand{\myglobal}[2]{\newcommand{#1}{{\color{\colourglobal}#2}}}
\myword{\word}{w}
\myword{\words}{\mathbf{w}}
\myword{\alphabet}{\Sigma}
\myword{\eos}{\mathrm{eos}}
\myword{\alphabeteos}{\overline{\alphabet}}
\myunnormed{\alphabetkeep}{\mathcal{D}}
\myunnormed{\alphabetkeeps}{\boldsymbol{\alphabetkeep}}

\myunnormed{\pythiamodel}{\mathcal{M}}

\newcommand{\alphabetkeeptopk}{\alphabetkeep_{\scaleto{\mathrm{\topkk=k'}}{4pt}}}
\newcommand{\alphabetkeeptopkwithval}[1]{\alphabetkeep_{\scaleto{\mathrm{\topkk=#1}}{4pt}}}
\newcommand{\alphabetkeeptopp}{\alphabetkeep_{\scaleto{\mathrm{\toppp=\pi'}}{4pt}}}
\newcommand{\alphabetkeeptoppwithval}[1]{\alphabetkeep_{\scaleto{\mathrm{\toppp=#1}}{4pt}}}

\myunnormed{\normconstant}{\mathcal{Z}}
\myunnormed{\decdistunormed}{\widehat{p}_{u}}
\mylocal{\decdistsubscript}{\alpha}
\mylocal{\decdist}{p_\alpha}
\myglobal{\decdistglobalsubscript}{\gamma}
\myglobal{\decdistglobal}{p_\gamma}
\mymacro{\topkname}{\mathrm{topk}}
\mymacro{\toppname}{\mathrm{topp}}
\mymacro{\topkk}{k}
\mymacro{\toppp}{\pi}
\mymacro{\vtheta}{\boldsymbol{\theta}}
\newcommand{\ptheta}{p_{\vtheta}}
\newcommand{\powerset}{\mathcal{P}}
\newcommand{\one}{\mathbbm{1}}

\mylocal{\constlocalnorm}{c_\alpha}
\myglobal{\constglobalnorm}{c_\gamma}
\myglobal{\constglobalnormcontext}{\widetilde{c}_\gamma}

\newcommand{\kl}{\mathrm{KL}}
\newcommand{\emptystring}{\emptyset}
\mymacro{\maxlength}{T}

\newcommand{\defeq}{\mathrel{\stackrel{\textnormal{\tiny def}}{=}}}
\newcommand{\mathcomment}[1]{\text{\textcolor{gray}{#1}}}

\mymacro{\imhstr}{\mathrm{imh}}
\newcommand{\pimh}{p_{\scaleto{\imhstr}{4pt}}}
\newcommand{\cmarkmath}{\text{\cmark}}
\newcommand{\imhcurrentstr}{\words}
\newcommand{\imhproposedstr}{\words'}
\newcommand{\puniform}{\texttt{Uniform}}

\newcommand{\iterations}{N}
\newcommand{\iterationslower}{n}

\title{Local and Global Decoding in Text Generation}

\newcommand{\luxembourgid}{{\includegraphics[scale=0.05]{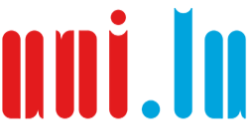}}}
\newcommand{\ethid}{{\includegraphics[scale=0.028]{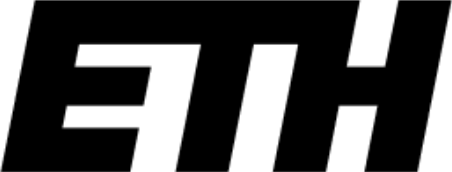}}}
\newcommand{\luxembourgprofemailadress}[1]{\href{mailto:#1@uni.lu}{\texttt{#1}}}
\newcommand{\luxembourgstudemailadress}[1]{\href{mailto:#1@gmail.com}{\texttt{#1}}}
\newcommand{\ethemailadress}[1]{\href{mailto:#1@inf.ethz.ch}{\texttt{#1}}}

\author{
Daniel Gareev,\textsuperscript{\luxembourgid} 
Thomas Hofmann,\textsuperscript{\ethid} 
Ezhilmathi Krishnasamy,\textsuperscript{\luxembourgid} 
Tiago Pimentel\textsuperscript{\ethid} \\
  $^\luxembourgid$University of Luxembourg, ~\,~  \textsuperscript{\ethid}ETH Z\"urich \\
  \luxembourgstudemailadress{daniel.gareev}@\texttt{gmail.com},~\,~
  \luxembourgprofemailadress{ezhilmathi.krishnasamy}@\texttt{uni.lu} \\
  \{%
  \ethemailadress{thomas\!.\!hofmann},
  \ethemailadress{tiago\!.\!pimentel}%
  \}@\texttt{inf\!.\!ethz\!.\!ch}
}

\begin{document}
\maketitle
\begin{abstract}
Text generation, a key component in applications such as dialogue systems, relies on decoding algorithms that sample strings from a language model distribution.
Traditional methods, such as top-$\topkk$ and top-$\toppp$, apply local normalisation to the model's output distribution, which can distort it.
In this paper, we investigate the effect of this distortion by introducing globally-normalised versions of these decoding methods.
Additionally, we propose an independent Metropolis-Hastings algorithm to approximate sampling from globally-normalised distributions without explicitly computing them.
Our empirical analysis compares the performance of local and global normalisation across two decoding algorithms (top-$\topkk$ and top-$\toppp$) with various hyperparameters, using Pythia language models.
Results show that, in most configurations, global decoding performs worse than the local decoding version of the same algorithms---despite preserving the distribution's integrity.
Our results suggest that distortion is an important feature of local decoding algorithms.

\vspace{5pt}
\begin{tblr}{colspec = {Q[c,m]|X[l,m]}, stretch = 0}
    \cincludegraphics[width=1.2em, keepaspectratio]{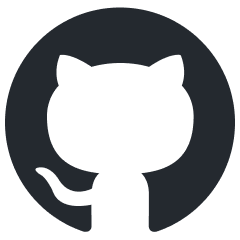}
     & \setstretch{.5}\href{https://github.com/lowlypalace/global-decoding}{{\makesf{lowlypalace/global-decoding}}} \\
\end{tblr}
\vspace{-5pt}
\end{abstract}

\newcommand{\defn}[1]{\textbf{#1}}

\section{Introduction}

Text generation is increasingly used in everyday applications, serving as a key component in dialogue agents like GPT-4 \cite{achiam2023gpt}.
Given a pre-trained language model (LM), text generation typically involves using \defn{decoding algorithms},\footnote{See \citet{welleck2024decoding} for a review.} 
which extract a string $\words$ from a language model $\ptheta$.
In open-ended text generation, these algorithms are usually stochastic, allowing users to \emph{sample} strings.
Importantly, the goal of text generation is to produce \defn{high-quality strings}, and a string's quality does not necessarily align with the probability mass assigned to it by a model \citep{holtzman2020curious,zhang-etal-2021-trading,meister-etal-2022-high}.

\begin{figure}[t]
    \centering
    \begin{tikzpicture}
        \node[anchor=north west] (image) at (0,0) {\includegraphics[trim={.8cm 1.4cm 1.8cm 1.6cm},clip,width=\columnwidth]{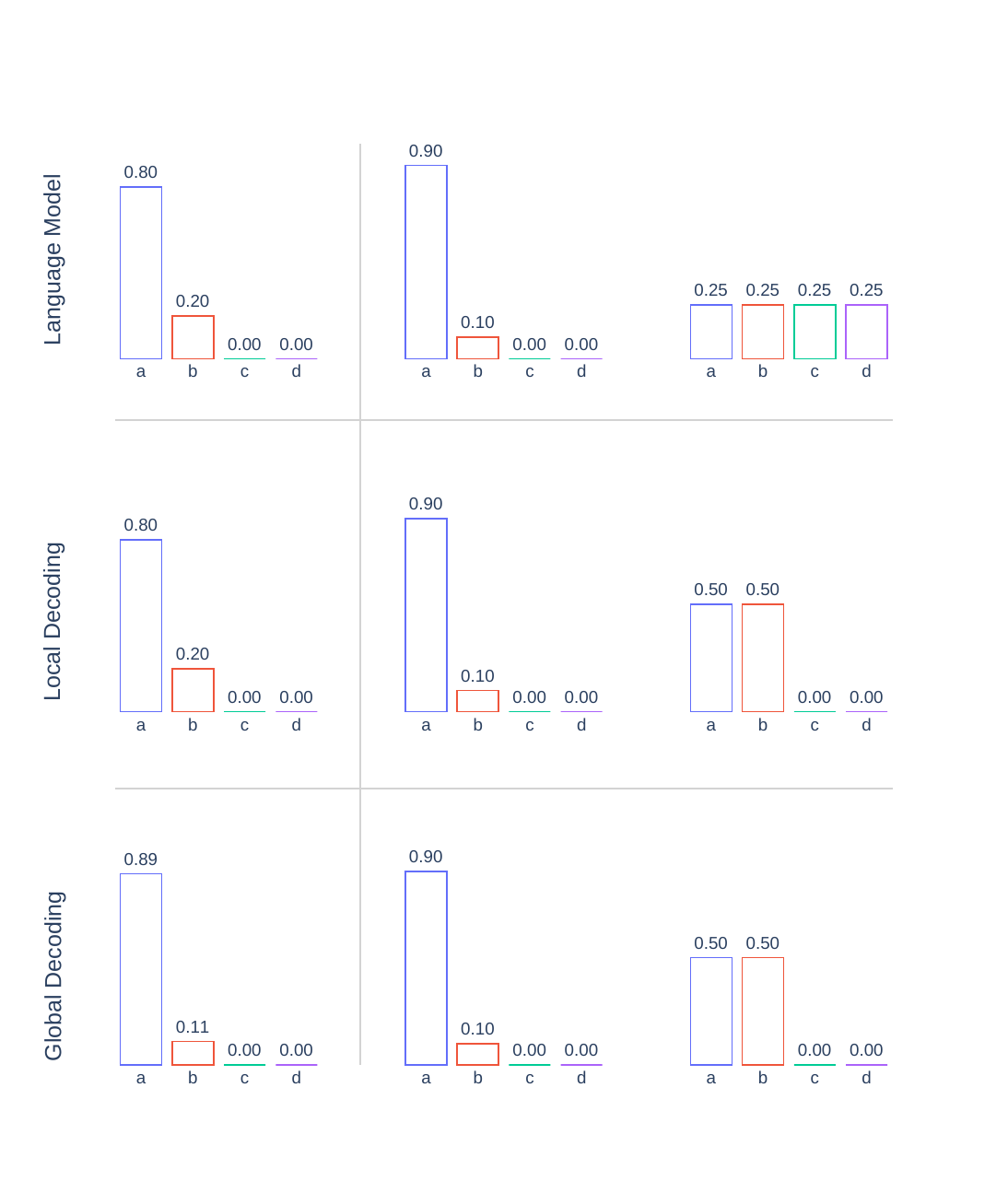}};
        \vspace{-20pt}

        \node at (1.7,-0.2) { \scalebox{.7}{ $p(\word_1 | "")$}};
        \node at (4.2,-0.2) { \scalebox{.7}{$p(\word_2 | a)$}};
        \node at (6.7,-0.2) {\scalebox{.7}{$p(\word_2 | b)$}};

        \draw [lightgray, dotted, thin] (6.3,-1.3) rectangle (7.7,-0.6);
        \node at (7,-0.8) { \scalebox{.55}{$\ptheta(ab) = 0.08$}};
        \node at (7,-1.1) { \scalebox{.55}{$\ptheta(ba) = 0.05$}};

        \draw [lightgray, dotted, thin] (6.3,-4) rectangle (7.7,-3.3);
        \node at (7,-3.5) {\scalebox{.55}{$\decdist(ab) = 0.08$}};
        \node at (7,-3.8) {\scalebox{.55}{$\decdist(ba) = 0.10$}};

        \draw [lightgray, dotted, thin] (6.2,-7.2) rectangle (7.7,-6.5);
        \node at (6.95,-6.7) { \scalebox{.55}{$\decdistglobal(ab) = 0.089$}};
        \node at (6.95,-7) { \scalebox{.55}{$\decdistglobal(ba) = 0.055$}};

        \draw [decorate,decoration={brace,amplitude=5pt,mirror,raise=2pt},yshift=-0.9cm]
            (0.6,-8.3) -- (2.9,-8.3) node [black,midway,yshift=-0.4cm] {\scalebox{.7}{$t=1$}};

        \draw [decorate,decoration={brace,amplitude=5pt,mirror,raise=2pt},yshift=-0.9cm]
            (3,-8.3) -- (7.7,-8.3) node [black,midway,yshift=-0.4cm] {\scalebox{.7}{$t=2$}};

    \end{tikzpicture}
    \vspace{-21pt}
    \caption{A simple language model over a four-symbol alphabet $\{a, b, c, d\}$ under local and global decoding. For this model, the probability of $ab$ is higher than $ba$; however, the opposite is true for local decoding. In this example, top-$\topkk$ is $2$ and maximum lengths are $T=2$.}
    \label{fig:distortion_example}
    \vspace{-10pt}
\end{figure}
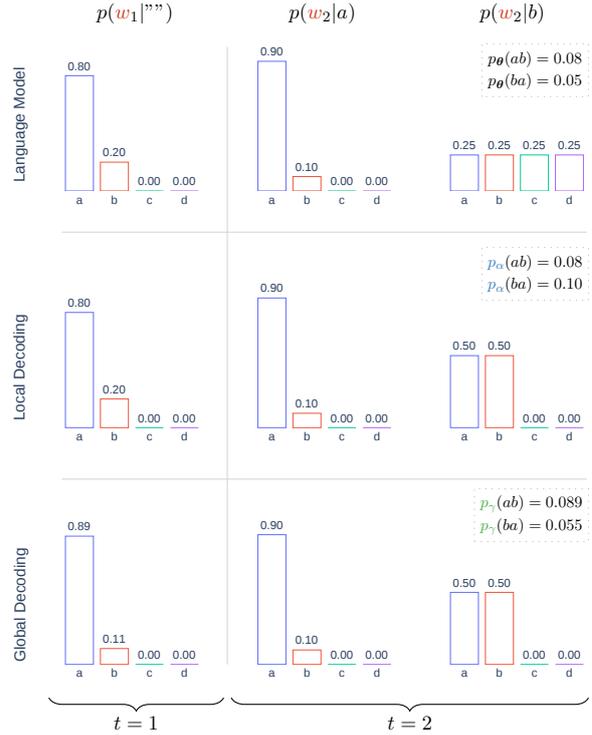

Over recent years, several decoding algorithms have been proposed \citep[e.g., top-$\topkk$, top-$\toppp$;][]{fan-etal-2018-hierarchical,holtzman2020curious,basu2021mirostat,hewitt-etal-2022-truncation,meister-etal-2023-locally}.
Most of these methods operate in two steps: (i) they first prune the model's output, assigning zero probability to a large set of strings, and (ii) they then sample from this pruned distribution.
However, before sampling, the distribution must be renormalised, i.e., its values must be adjusted to sum to one.
This renormalisation is typically performed \emph{locally} for each context, meaning that each conditional $\ptheta(\cdot \mid \words_{<t})$ is renormalised independently; a process known as \defn{local normalisation}.
Importantly, by renormalising contexts independently, local normalisation can distort 
 distribution $\ptheta$ (see \cref{fig:distortion_example}).\looseness=-1

This paper explores the effects of such distortion on decoding algorithms. 
Specifically, we propose globally-normalised versions of classical decoding methods such as top-$\topkk$ and top-$\toppp$.
Globally-normalised methods prune the model's output identically to the locally-normalised version. 
However, instead of renormalising each context separately, they normalise the entire distribution $\ptheta(\cdot)$ at once; a process known as \defn{global normalisation}.
Unlike local normalisation, global normalisation does not distort the distribution, allowing us to examine how distortion affects decoding performance.

However, global normalisation is generally intractable as it requires computing a sum over an infinite set of strings.
To address this issue, we use 
a Markov chain Monte Carlo (MCMC) method to sample strings, adapting the independent Metropolis-Hastings (IMH) algorithm to the decoding setting.
As IMH only requires unnormalised probabilities, 
this method allows us to approximately sample from the globally-normalised distribution without explicitly computing it.

In our experiments, we compare locally and globally normalised versions of two decoding algorithms: top-$\topkk$ and top-$\toppp$.
We run these algorithms on Pythia models (ranging from 70m to 2.8b in size) with several hyperparameter configurations (8 settings for each algorithm; with $\topkk$ spanning from $5$ to $10{,}000$ and $\toppp$ from $0.01$ to $0.99$).
Our results show that globally-normalised methods generally perform worse than their locally-normalised counterparts, as evaluated by MAUVE scores.
Additionally, our results suggest that local normalisation leads to longer, less repetitive and overall higher-quality text.
We conclude that the distortion introduced by local decoding is an important component contributing to its performance.

\section{Language Modelling}
A language model, denoted as $\ptheta(\words)$ with parameters $\vtheta$, defines a probability distribution over the set of all finite strings $\words = \langle\word_1, \word_2, \dots, \word_{|\words|} \rangle \in \alphabet^*$, where $\alphabet$ represents an alphabet of subword units.
These models are generally defined autoregressively as:
\begin{align}
    \ptheta(\words) = \prod_{t=1}^{|\words| + 1} \ptheta(\word_t \mid \words_{<t})
\end{align}
Here, $\word_{|\words|+1}$ represents a special end-of-sequence symbol ($\eos \notin \alphabet$).
While $\ptheta(\words)$ represents a global distribution over $\alphabet^*$, the conditionals $\ptheta(\word_t \mid \words_{<t})$ describe local distributions over an $\eos$-augmented set of subword units $\alphabeteos \defeq \alphabet\cup\{\eos\}$.
Importantly, sampling from $\ptheta(\words)$ is straightforward, as it reduces to sampling iteratively from $\ptheta(\word_t \mid \words_{<t})$ until $\eos$ is selected.

In practice, users of language models typically impose a maximum string length $\maxlength$ to constrain the output of $\ptheta$.
To facilitate the upcoming discussion, we formally define such $\maxlength$-maxlength language models here.
\begin{defin}\label{defn:maxlength_lang_model}
    A \defn{$\boldsymbol{\maxlength}$-maxlength language model} is a LM for which any string $\words$ longer than $\maxlength$ has a probability of zero. Formally:
    \begin{align}
        |\words| > \maxlength \implies \ptheta(\words) = 0
    \end{align}
    This implies that for any $\maxlength$-length prefix $\words_{\leq \maxlength}$ with non-zero probability, $\ptheta(\eos \mid \words_{\leq \maxlength}) = 1$.
\end{defin}

\section{Decoding Algorithms}

In this section, we first introduce both local and global decoding algorithms.
We then discuss how these distributions compare to each other.

\subsection{Local Decoding}

Most \defn{decoding algorithms} define a \defn{pruning function}
$\alphabetkeep: \alphabet^* \to \powerset(\alphabeteos)$ which, given a prefix $\words_{<t}$, returns a subset of $\alphabeteos$ to retain in the distribution.\footnote{$\powerset(\alphabeteos)$ refers to the powerset of alphabet $\alphabeteos$.} 
Subwords which are not in $\alphabetkeep(\words_{<t})$ are then assigned a probability of zero.\footnote{Our description of local decoding algorithms is inspired by the sampling adapters framework \citep{meister-etal-2023-efficacy}.}
The pruning function is used to modify the LM's output distribution as follows:\looseness=-1%
\begin{align}
    &\decdistunormed(\word \mathop{\mid} \words_{<t}) \mathop{=} \ptheta(\word \mathop{\mid} \words_{<t})\, \one\{\word \mathop{\in} \alphabetkeep(\words_{<t})\} \label{eq:decision_rule}
\end{align}
Here, $\one\{\cdot\}$ is an 
indicator function that returns 1 if the condition holds and 0 otherwise.
Thus, $\decdistunormed(\word \mathop{\mid} \words_{<t})$ represents an \defn{unnormalised pruned distribution}, where each subword $\word \mathop{\notin} \alphabetkeep(\words_{<t})$ is re-assigned zero probability.
A \defn{local decoding algorithm} then normalises this distribution as:\footnote{We assume that $\alphabetkeep$ is defined such that the probabilities are well-defined for all contexts, i.e., $\forall \words_{<t} \in \alphabet^{*}$, we have $\sum_{\word' \in \alphabeteos} \decdistunormed(\word' \mid \words_{<t}) > 0$.}
\begin{subequations} \label{eq:local_decoding}
\begin{align}
    &\decdist(\word \mid \words_{<t}) = \frac{\decdistunormed(\word \mid \words_{<t})}{\sum_{\word' \in \alphabeteos} \decdistunormed(\word' \mid \words_{<t})} \\
    &\decdist(\words) = \prod_{t=1}^{|\words|+1} \decdist(\word_t \mid \words_{<t})
\end{align}
\end{subequations}
where we use subscript $\decdistsubscript$ to denote these distributions as locally normalised.
As with the original model $\ptheta(\words)$, sampling from $\decdist(\words)$ is straightforward: one can simply iteratively sample from the conditional $\decdist(\cdot \mid \words_{<t})$ at each time step.
Several popular decoding algorithms can be instantiated by defining the function $\alphabetkeep$.
We provide two examples next.

\begin{defin} \label{defn:topk_decoding}
    \defn{Top-k decoding} \citep{fan-etal-2018-hierarchical} is a local decoding algorithm with pruning function:%
    \begin{align}
        \alphabetkeeptopk(\words_{<t}) = &\argmax_{\alphabetkeep' \subseteq \alphabeteos} \sum_{\word \in \alphabetkeep'} \ptheta(\word \mid \words_{<t}),\\ 
        &\mathrm{s.t.}\, |\alphabetkeep'| = k' \nonumber
    \end{align}
\end{defin}

\begin{defin} \label{defn:topp_decoding}
    \defn{Top-$\boldsymbol{\pi}$ decoding} \citep{holtzman2020curious} is a local decoding algorithm with pruning function:
    \begin{align}
        \alphabetkeeptopp(\words_{<t}) = &\argmin_{\alphabetkeep' \subseteq \alphabeteos} |\alphabetkeep'|, \\
        &\mathrm{s.t.}\sum_{\word \in \alphabetkeep'} \ptheta(\word \mid \words_{<t}) \geq \pi' \nonumber
    \end{align}
\end{defin}

\subsection{Global Decoding}

Most decoding algorithms operate over $\ptheta(\words)$ as described above: 
they select a subset of strings to assign zero probability and then re-normalise this distribution \emph{locally}.
However, this re-normalisation process can also be done \emph{globally} instead.
We define \defn{global decoding algorithms} as:%
\begin{subequations} \label{eq:global_decoding}
\begin{align}
    &\decdistunormed(\words) = \prod\limits_{t=1}^{|\words|+1} \decdistunormed(\word_t \mid \words_{<t}) \\
    &\decdistglobal(\words) = \frac{\decdistunormed(\words)}{\sum_{\words' \in \alphabet^*} \decdistunormed(\words')} \label{eq:normalized_global}
\end{align}
\end{subequations}
where $\decdistglobalsubscript$ marks this distribution as globally normalised.
Unlike local normalisation, global normalisation does not \defn{distort} the distribution beyond the pruning process.
For any unpruned string, we have $\decdistglobal(\words) \propto \ptheta(\words)$.

\mymacro{\variationalfamily}{\mathcal{V}_{\maxlength}}

\subsection{Local vs.\ Global Decoding}
\label{sec:local_global_decoding}

The distributions $\decdistglobal(\words)$ and $\decdist(\words)$ can differ significantly.
They might, for instance, rank strings in the opposite order;
given two strings $\words$ and $\words'$, we might have: 
$\decdistglobal(\words) < \decdistglobal(\words')$ and
$\decdist(\words) > \decdist(\words')$.
Moreover, the Kullback-Leibler (KL) divergence between them can be arbitrarily large.
We now prove two theorems about these distributions.

\begin{restatable}{theorem}{kllowerbound}
\label{lemma:kl_lower_bound}
    Let $\variationalfamily$ be a set that includes all $\maxlength$-maxlength language models $\ptheta(\words)$ (see \cref{defn:maxlength_lang_model}).
    There exist language models $\ptheta \in \variationalfamily$, for which the top-$\topkk$ and top-$\toppp$ decoding versions $\decdistglobal(\words)$ and $\decdist(\words)$ have $\kl$s bounded below as:
    \begin{subequations}
    \begin{align}
        \sup_{\ptheta \in \variationalfamily}
        \kl\big(\decdistglobal(\words) \mid\mid \decdist(\words)\big)
        \in \Omega(\maxlength) \\
        \sup_{\ptheta \in \variationalfamily}
        \kl\big(\decdist(\words) \mid\mid \decdistglobal(\words)\big)
        \in \Omega(\maxlength)
    \end{align}
    \end{subequations}
    where $\Omega$ represents a lower bound in asymptotic notation.
\end{restatable}
\begin{proof}
    See \cref{app:proofkllowerbound}.
\end{proof}

This theorem states that there is at least one choice of $\ptheta(\words)$ for which the $\kl$ between local and global decoding distributions grows linearly with the maximum string length $\maxlength$.
Therefore, these distributions can differ considerably.
In the next theorem, we also provide an upper bound for the divergence between the two distributions.

\mymacro{\pmin}{p_{\mathrm{min}}}

\begin{restatable}{theorem}{klupperboundboth}
\label{lemma:kl_upper_bound_both}
    Let $\pmin$ be the minimum probability retained at each time step by either top-$\topkk$ (whose $\pmin = \frac{\topkk}{|\alphabeteos|}$) or top-$\toppp$ (whose $\pmin = \toppp$). 
    When using either of these decoding algorithms, both forward and reverse $\kl$s between $\decdistglobal(\words)$ and $\decdist(\words)$ are upper bounded by:
    \begin{subequations}
    \begin{align}
        \kl\big(\decdistglobal(\words) \mid\mid \decdist(\words)\big) \leq 
        \maxlength\,\log \frac{1}{\pmin}, \\
        \kl\big(\decdist(\words) \mid\mid \decdistglobal(\words)\big) \leq 
        \maxlength\,\log \frac{1}{\pmin}
    \end{align}
    \end{subequations}
    where $\ptheta(\words)$ is a $\maxlength$-maxlength language model.
\end{restatable}
\begin{proof}
    See \cref{app:proofklupperbound_both}.
\end{proof}

This second theorem upper-bounds the $\kl$ between local and global decoding in terms of $\pmin$, the minimum probability kept by the pruning function for any context $\words_{<t}$.
Notably, top-$\toppp$ typically uses hyperparameters that result in relatively large $\pmin$ values (e.g., $\toppp \in [.8, .9, .95]$), which may explain why these settings are preferred—choosing smaller $\toppp$ could lead to larger distortions. 
In contrast, top-$\topkk$ is usually applied with hyperparameters resulting in smaller $\pmin$ values (e.g., $\topkk \in [5, 10, 100]$).
Could its reduced distortion be what gives top-$\toppp$ an advantage over top-$\topkk$?
Our experiments aim to explore this and related questions.

\paragraph{A Note on Prior Work.}
Previous research has also explored differences between local and global normalisation in text generation.
For instance, \citet{goyal-etal-2019-empirical} note that---while locally and globally normalised language models are equally expressive in general---local normalisation may impose constraints on the strings returned by beam search.
Most similar to our work, \citet{zhang-etal-2021-trading} introduce a globally normalised version of temperature sampling and compare it to its locally normalised version.\footnote{We note this experiment is only present in the paper's arXiv version but not in the HumEval workshop version.}
However, their comparison between global and local decoding is less direct as they also impose a threshold on the maximum allowable probability of a sampled string.
To the best of our knowledge, this paper is the first to systematically compare global and local normalisation in text generation, focusing on the widely used top-$\topkk$ and top-$\toppp$ decoding methods.

\section{Sampling Strings with Independent Metropolis--Hastings}
\label{sec:imh}

As noted in our introduction, sampling strings directly from $\decdistglobal(\words)$ is generally intractable.
However, given access to the unnormalised $\decdistunormed(\words)$, Markov chain Monte Carlo algorithms allow us to approximate this sampling process.
Several such methods exist in the controllable text generation literature \citep[\textit{inter alia}]{miao2019cgmh,amini2023structured,forristal-etal-2023-block,lew2023sequential,du2023principled}.\footnote{
Non-MCMC-based methods also exist for controllable text generation \citep[e.g.,][]{yang-klein-2021-fudge}.
}
Here, we propose a new approach based on independent Metropolis-Hastings.
The method proceeds as follows: we first sample an initial string $\words^{(1)}$ from a \defn{proposal distribution} $\pimh(\words)$.
Then, over $\iterations$ iterations, we sample new proposal strings $\words'\sim\pimh(\words)$ and decide whether to accept the new string based on an \defn{accept--reject distribution} $\pimh(\cmarkmath \mathop{\mid} \imhproposedstr, \imhcurrentstr^{(\iterationslower-1)})$.
These steps are summarised in the following equations:
\begin{subequations} \label{eq:imh_steps}
\begin{align}
    &\words' \sim \pimh(\words) \\
    &x \sim \puniform([0, 1]) \\
    &\words^{(k)} \!=\! \left\{\!\!\!\begin{array}{lr}
         \words' & \!\!\!\!\!\!\!\!\!\! \texttt{if}\,x \leq \pimh(\cmarkmath \mathop{\mid} \imhproposedstr, \imhcurrentstr^{(\iterationslower-1)}) \\
         \words^{(\iterationslower-1)} & \texttt{else}
    \end{array}\right.\!\!\!\!
\end{align}
\end{subequations}
Finally, we take string $\words^{(\iterations+1)}$ as our sample.
Notably, if we define the accept--reject distribution as:\looseness=-1%
\begin{align}
    \pimh(\cmarkmath \mathop{\mid} \imhproposedstr, \imhcurrentstr) \mathop{\defeq} 
     \min\left(\!\! 1, \frac{\decdistunormed(\imhproposedstr)\, \pimh(\imhcurrentstr)}{\decdistunormed(\imhcurrentstr)\, \pimh(\imhproposedstr)}\right)
\end{align}
this procedure approximates sampling from $\decdistglobal(\words) \mathop{\propto} \decdistunormed(\words)$ for large enough $\iterations$.
The convergence rate of IMH depends on how closely $\pimh(\words)$ matches $\decdistglobal(\words)$ \citep{wang2021exact}.
As we do not have direct access to $\decdistglobal(\words)$, we instead use its locally normalised version as the proposal distribution:
\begin{align}
    &\pimh(\words) \defeq \decdist(\words) 
\end{align}
We present a pseudo-code of this method in \cref{alg:independent_metropolis_hastings}.

\begin{figure}
    \centering
\begin{minted}[escapeinside=||,mathescape=true]{python}
def ind_metropolis_hastings(|$\decdist$|, |$\decdistunormed$|, |$\iterations$|):
    |$\pimh$| = define_accept_reject(|$\decdist$|, |$\decdistunormed$|)
    |$\words$| = |$\decdist$|.sample()
    for |$\iterationslower$| in range(|$\iterations$|):
        |$\words^{\prime}$| = |$\decdist$|.sample()
        |$x$| = random.uniform(low=0.0, high=1.0)
        if u |$\leq \pimh(\cmarkmath \mid \words^{\prime}, \words)$|:
            |$\words = \words^{\prime}$|
    return |$\words$|
\end{minted}
    \vspace{-13pt}
    \caption{Pseudo-code for independent Metropolis--Hastings sampling.}
    \label{alg:independent_metropolis_hastings}
    \vspace{-7pt}
\end{figure}

\paragraph{A Note on Prior Work.}
Previous methods in controlled generation \citep[e.g.,][]{miao2019cgmh} similarly
sample strings according to \cref{eq:imh_steps}, with the primary difference being their definition of $\pimh(\words)$.
These methods are based on Metroplis-Hastings and define a \emph{conditional} proposal distribution $\pimh(\words \mid \words^{(\iterationslower-1)})$.
In contrast, we propose an \emph{independent} Metropolis-Hastings method for multiple reasons.
First, we can batch sample multiple strings from $\pimh(\words)$ simultaneously.
Second, we can compute both $\decdistunormed(\words)$ and $\pimh(\words)$ for a string in a single pass through the model, and simultaneously for multiple strings in a batch.
Third, some choices of decoding algorithms (e.g., top-$\topkk$ with small $\topkk$) assign zero probability to most strings in $\alphabet^*$. Depending on the choice of $\pimh(\words \mid \words^{(\iterationslower-1)})$ in prior work, some sequences may become unreachable from others, potentially breaking the convergence guarantees of Metropolis-Hastings.

\section{Experimental Setup}

As discussed above, our experiments are designed to compare local and global decoding algorithms, specifically evaluating the effects of local normalisation's distortion on decoding performance.
To achieve this, we generate a set of strings using pre-trained language models---particularly Pythia \citep{biderman2023pythia}---with either local or global decoding strategies.
Our experimental setup is flexible and can be applied to any local decoding algorithm. 
For this study, we focus on top-$\topkk$ and top-$\toppp$ decoding, exploring a broad range of configurations for each.
Below, we provide details on the key components of our experimental design, including the models we use, decoding algorithm configurations, IMH hyperparameters and evaluation protocols.

\begin{figure*}[t]
  \centering
  \begin{subfigure}[b]{\textwidth}
    \includegraphics[trim={.4cm .2cm 1.7cm 1.8cm},clip,width=\linewidth]{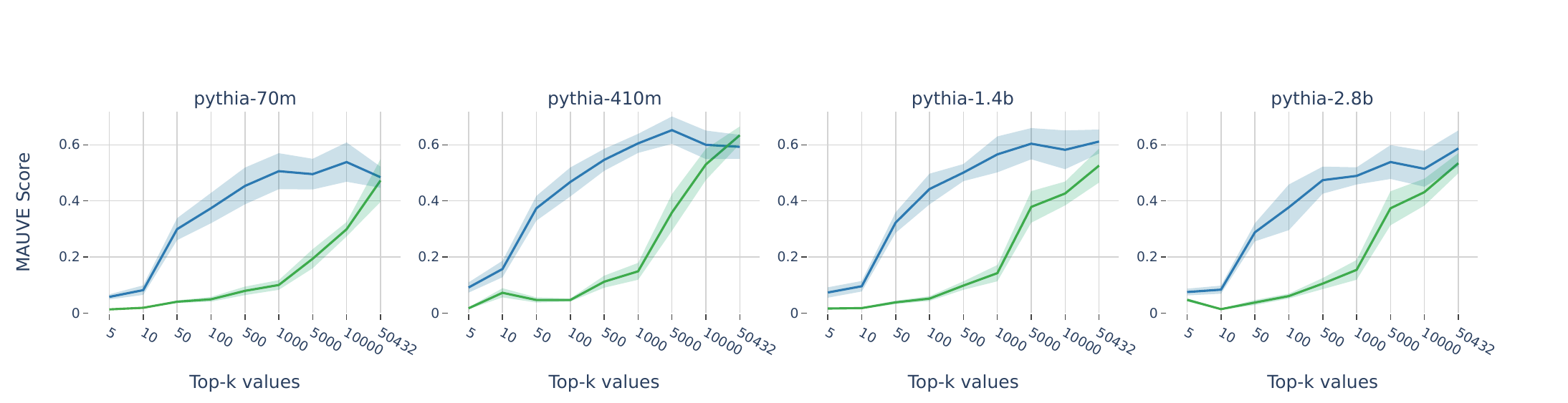}
  \end{subfigure}
  \begin{subfigure}[b]{\textwidth}
    \includegraphics[trim={.4cm 0 1.7cm 1.5cm},clip,width=\linewidth]{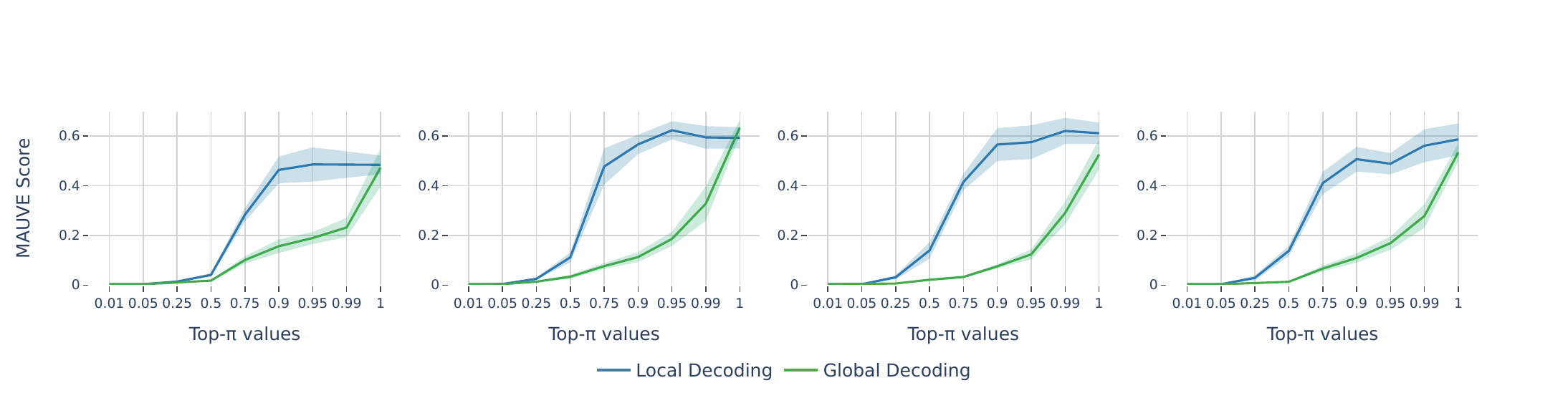}
  \end{subfigure}
  \vspace{-20pt}
  \caption{MAUVE evaluation scores when using local and global decoding with various top-$\topkk$ (Top) and top-$\toppp$ (Bottom) settings. 
Results are averaged over 10 runs, and 95\% confidence intervals are shown.}
  \label{fig:mauvescores}
\end{figure*}

\begin{figure*}[t]
  \centering
  \begin{subfigure}[b]{\textwidth}
    \includegraphics[trim={.2cm .2cm 1.7cm 1.5cm},clip,width=\linewidth]{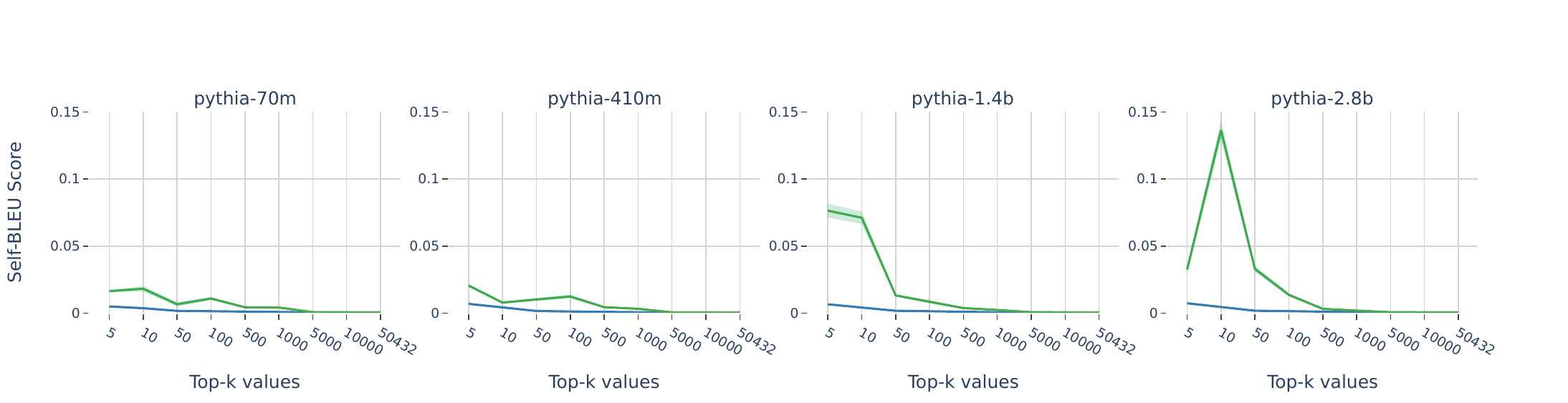}
  \end{subfigure}
  \begin{subfigure}[b]{\textwidth}
    \includegraphics[trim={.2cm 0 1.7cm 1.5cm},clip,width=\linewidth]{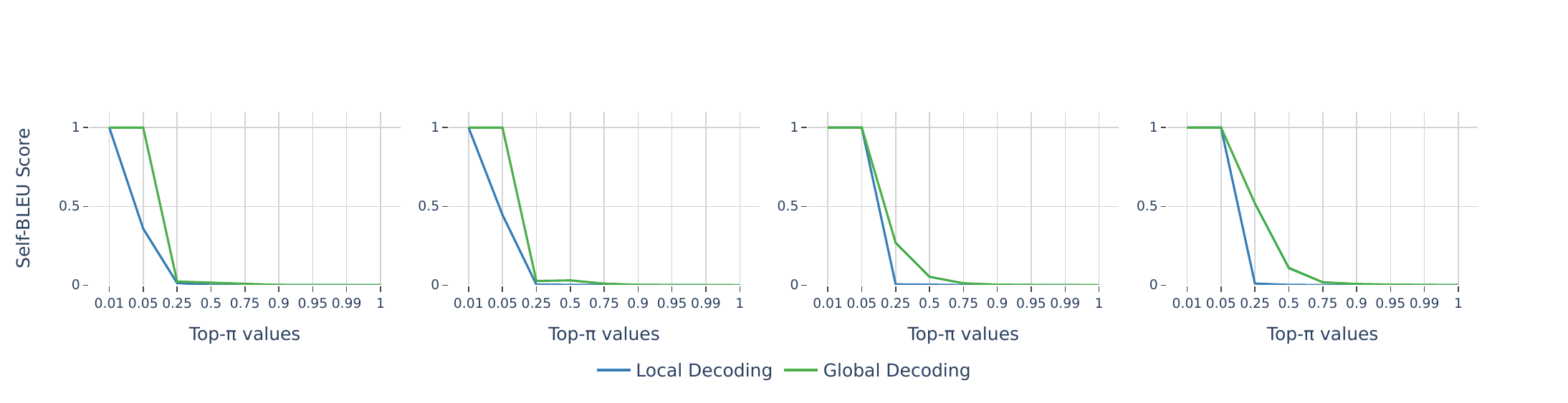}
  \end{subfigure}
  \vspace{-20pt}
  \caption{Self-BLEU evaluation scores when using local and global decoding with various top-$\topkk$ (Top) and top-$\toppp$ (Bottom) settings. 
Results are averaged over 10 runs, and 95\% confidence intervals are shown.}
  \label{fig:bleuscores}
\end{figure*}

\vspace{-2pt}
\subsection{Models}

For our experiments, we use Pythia \cite{biderman2023pythia}, a suite of open-source autoregressive (decoder-only) language models $\ptheta(\words)$.
Pythia models are well-suited for performing scientific experiments with controlled settings---such as studying decoding behaviours and analysing their impact on open-ended text generation quality.
Specifically, we use pythia-70m, pythia-410m, pythia-1.4b and pythia-2.8b.
All models are run in double precision to avoid numerical instability.

\vspace{-2pt}
\subsection{Decoding Configurations}

As mentioned above, we focus our experiments on top-$\topkk$ and top-$\toppp$.
We analyse several configurations of these methods.
We study top-$\topkk$ 
with $\topkk \in \{5,\allowbreak 10,\allowbreak 50, 100,\allowbreak 500,\allowbreak 1000,\allowbreak 5000,\allowbreak 10{,}000\}$ and
we explore top-$\toppp$ with
$\toppp \in \{0.01,\allowbreak 0.05,\allowbreak 0.25,\allowbreak 0.5,\allowbreak 0.75,\allowbreak 0.9,\allowbreak 0.95,\allowbreak 0.99\}$.
Additionally, we run experiments without applying any pruning strategy, which is equivalent to running top-$\topkk$ while keeping the entire vocabulary $\alphabetkeeptopkwithval{|\alphabet|}$ (composed of $50{,}432$ tokens), or top-$\toppp$ while keeping all probability mass $\alphabetkeeptoppwithval{1}$.
This serves as a baseline decoding strategy and corresponds to using $\alphabetkeep(\words_{<t}) = \alphabeteos$ in \cref{eq:decision_rule}, as no words are pruned by the decoding algorithm.

\vspace{-2pt}
\subsection{Text Sampling}

For text generation, Pythia models require at least one token in the input context.
We use the $\eos$ token as a prompt, which is a common practice for generating open-ended text continuations. 
We generate $100{,}000$ sequences for each pair of model--decoding configurations.
At each time step, tokens $\word_t$ are sampled from $p_\theta(\cdot \mid \words_{<t})$ until either the maximum length of $\maxlength = 512$ tokens is reached, or until the $\eos$ token is sampled.

\vspace{-2pt}
\subsection{IMH Hyperparameters}

To approximate sampling from  $\decdistglobal(\words)$, we use independent Metropolis--Hastings.
We run IMH algorithm $200$ times for each model and decoding configuration, using independent string samples for each.
Each IMH call is run for $\iterations=500$ iterations, thus requiring $100{,}000$ sequences in total (as $100{,}000 = 500 \cdot 200$).
We then select the last accepted sequence from each IMH call, resulting in 200 strings approximately sampled from $\decdistglobal$.

\subsection{Evaluation Methods}

We use MAUVE \cite{pillutla-etal:mauve:neurips2021} to evaluate the model-generated samples when using either local or global decoding. 
This metric measures the similarity between model-generated text and human-generated reference text. As human-generated reference text, we use the test set of the WebText dataset \cite{radford2019language}.
Higher MAUVE scores imply that model-generated strings are more similar to the human-generated reference samples, suggesting higher text quality.
MAUVE scores are reported in \cref{fig:mauvescores}, as well as in \cref{table:mauve} (in the appendix).\footnote{Although the values for $\alphabetkeeptopkwithval{|\alphabet|}$ and $\alphabetkeeptoppwithval{1.0}$ differ for local and global decoding, this is because the sequences for evaluation were sampled independently, even though the decoding methods coincide for this configuration.}
Additionally, we compute self-BLEU scores \cite{zhu2018texygen} to measure the diversity of generated texts.
Higher self-BLEU scores indicate lower diversity, meaning the generated samples are more similar to each other. 
Lower self-BLEU scores thus typically correspond to higher-quality text.
Self-BLEU scores are shown in \cref{fig:bleuscores}, as well as in \cref{table:bleu} (in the appendix).

\newcommand{\mauvelocal}{\scriptsize\textsc{Mauve\textsubscript{local}}\xspace}
\newcommand{\mauveglobal}{\scriptsize\textsc{Mauve\textsubscript{global}}\xspace}
\newcommand{\bleulocal}{\scriptsize\textsc{Bleu\textsubscript{local}}\xspace}
\newcommand{\bleuglobal}{\scriptsize\textsc{Bleu\textsubscript{global}}\xspace}

\subsection{Bootstrapping}

For each model-decoding configuration, we generate 100{,}000 sequence samples. These sequences are then (i) used to run IMH, producing 200 global decoding sequences, and (ii) subsampled without repetition to obtain 200 local decoding sequences. 
To compute confidence intervals, we apply bootstrapping to these samples. 
Specifically, we resample with replacement from the original 100{,}000 sequences to create a new set of 100,000 sequences. 
This resampling is done 10 times, after which we generate globally and locally decoded samples using the bootstrapped sets. 
We compute evaluation metrics (MAUVE and self-BLEU) based on these decoded sequences and report the mean values along with 95\% confidence intervals, derived from the 10 bootstrapped evaluation runs.

\section{Main Results}

In this section, we present and analyse our results. 
First, we evaluate locally and globally decoded text using Mauve and self-BLEU metrics.
Next, we examine the distortion caused by local normalisation and assess the impact of using IMH to sample from globally-normalised distributions.
Lastly, we provide a qualitative analysis of the generated text.

\subsection{Decoding Quality}

\Cref{fig:mauvescores} shows MAUVE scores for both local and global decoding.
As shown in this figure, local decoding algorithms generally achieve higher MAUVE scores across most configurations compared to global decoding algorithms. 
This implies that local normalisation produces more human-like text, with the difference particularly noticeable in the mid-range settings of  top-$\topkk$ or top-$\toppp$.
Notably, local and global decoding produce equivalent distributions when either $\topkk = |\alphabet|$, or $\toppp=1.0$. 
They are also equivalent when $\topkk=1$ or $\toppp \to 0$, where the pruning process retains only one string.
However, global decoding scores drop rapidly as $\topkk<|\alphabet|$ or $\toppp<1.0$, while local decoding scores even increase for some models.
This suggests that, while global normalisation preserves the overall distribution, local decoding's distortion may enhance text quality rather than degrade it.
Our results thus imply that these distortion artefacts introduced by local normalisation may improve text coherence and its resemblance to human-generated text.

\subsection{Text Repetitions}

\Cref{table:bleu} presents self-BLEU scores for text samples from both local and global decoding algorithms.
As shown, local decoding generally produces more diverse text, reflected by its lower self-BLEU scores.
Global normalisation thus appears to lead to \emph{less} coherent sequences with \emph{higher} repetition rates.
\begin{figure*}[t]
    \centering
    \includegraphics[trim={.1cm .2cm 1.5cm 2cm},clip,width=\linewidth]{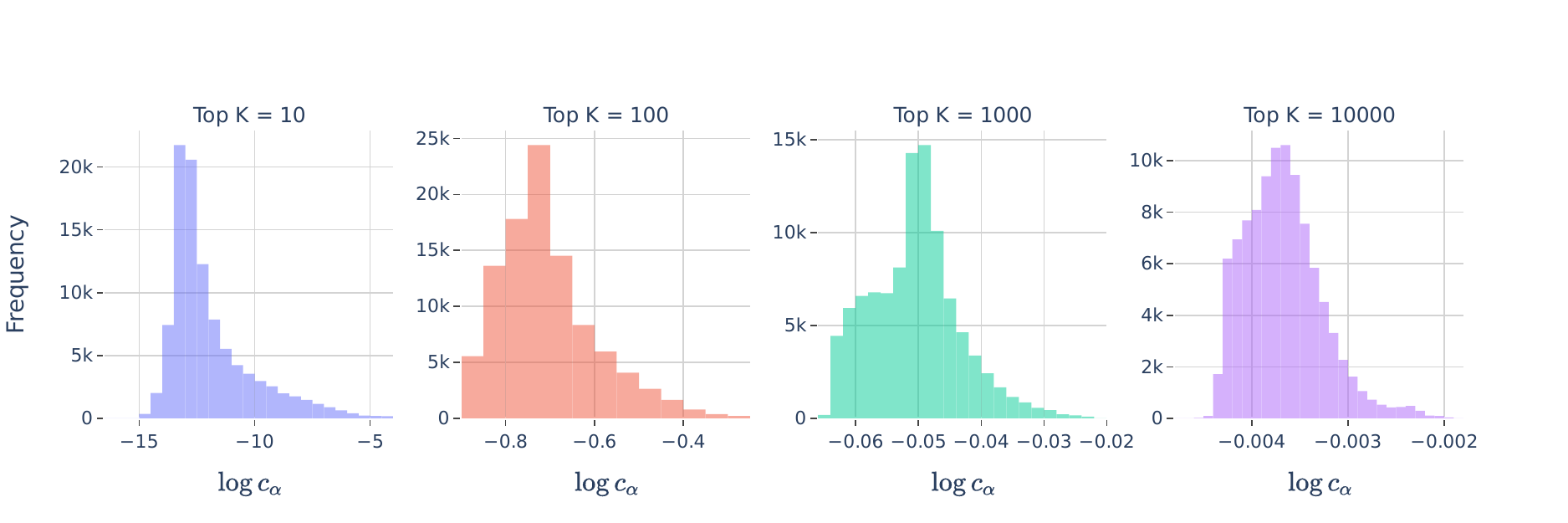}
    \caption{Histogram of local normalisation constants per sequence for pythia-2.8b when using different $\topkk$ values for top-$\topkk$.
    These constants $\constlocalnorm(\words)$ are defined in \cref{eq:seq_local_norm_const}.
    The $x$-axis is $\log$-scaled.}
    \label{fig:local_normalisation_constants}
\end{figure*}

\newcommand{\constlocalnormsequence}{\overline{\constlocalnorm}}

\subsection{Impact of Local Normalisation}

We now examine the distortion introduced by local normalisation, which highlights the differences between global and local decoding.
To do this, we present the sequence-level local normalisation constants in \cref{fig:local_normalisation_constants}.
These constants are the product of all subword-level local normalisation constants in a sequence, defined as:
\begin{align}\label{eq:seq_local_norm_const}
    \constlocalnormsequence(\words) \defeq 
    \prod_{t=1}^{|\words|} \underbrace{\sum_{\word' \in \alphabeteos} \decdistunormed(\word'\mid\words_{<t})}_{\constlocalnorm(\words_{<t}) \text{ in \cref{defn:local_constant}}}
\end{align}
Since these constants vary across different strings, they distort the distribution:
$\decdist(\words) = \frac{\decdistunormed(\words)}{\constlocalnormsequence(\words)} \not\propto \ptheta(\words)$.
\Cref{fig:local_normalisation_constants} shows that these constants are more uniform for top-$\topkk$ with large $\topkk$ values (note that the $x$-axis have different scales and are logged).\footnote{Also note that $\log \decdist(\words) = \log \decdistunormed(\words) - \log \constlocalnormsequence(\words)$.} 
This suggests that, as discussed in \cref{sec:local_global_decoding},  locally normalised distributions with these configurations are less distorted.
Combined with the results in \cref{fig:mauvescores}, our findings suggest that a \emph{moderate} level of distortion, may improve decoding performance.

\begin{figure}[t]
    \centering
    \includegraphics[trim={.6cm .2cm 1.7cm 2cm},clip,width=\linewidth]{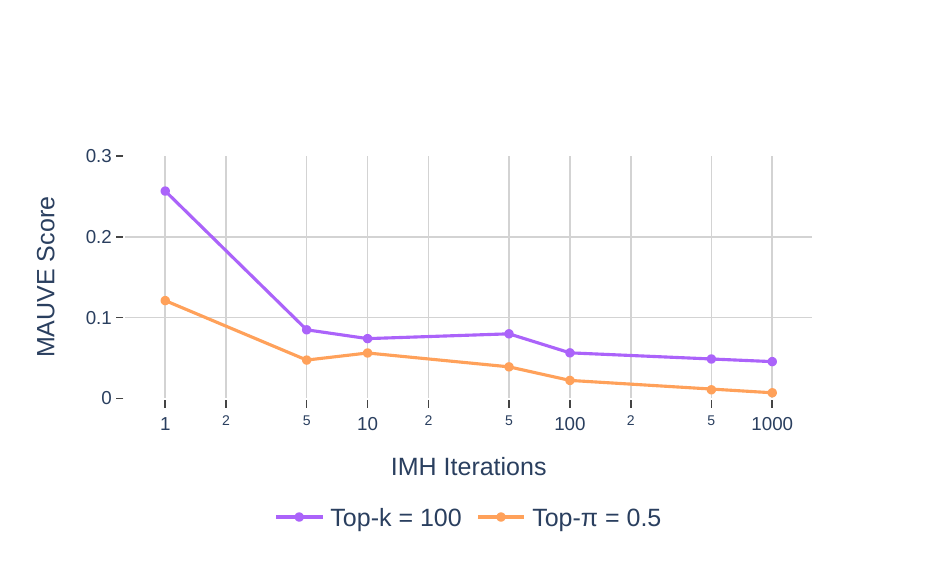}
    \caption{MAUVE Scores as the number of IMH iterations $\iterations$ changes. We evaluate global decoding with top-$\topkk$ ($\topkk = 100$) and top-$\toppp$ ($\toppp = 0.5$).
    We use 200 strings for each evaluation, except when $\iterations = 1000$, in which case we use only 100 strings. Strings were generated using pythia-2.8b.
    }
    \label{fig:mcmc_iterations_plot}
\end{figure}

\subsection{Effect of IMH Sampling on Decoding}

In this section, we examine the impact of IMH sampling on our results.
As discussed in \cref{sec:imh}, IMH offers a way to \emph{approximately} sample from distribution $\decdistglobal(\words)$, using only unnormalised scores $\decdistunormed(\words)$.
This method is approximate and converges to the target distribution as the number of iterations increases (i.e., as $\iterations\to\infty$).
On the other hand, running IMH with $\iterations=1$ is equivalent to sampling from the proposal distribution, which in our case is the local decoding distribution, $\decdist(\words)$.
In \cref{fig:mcmc_iterations_plot}, we show how MAUVE scores change with the number of IMH iterations used.
As this figure illustrates, increasing $\iterations$ from 1 to 10 significantly affects the results, but scores stabilise for $\iterations > 100$.
Therefore, even with approximate samples, we believe that our overall results are representative of the global distribution $\decdistglobal(\words)$.

\begin{figure*}[t]
    \centering
    \includegraphics[trim={.2cm .2cm 1.7cm 2cm},clip,width=\linewidth]{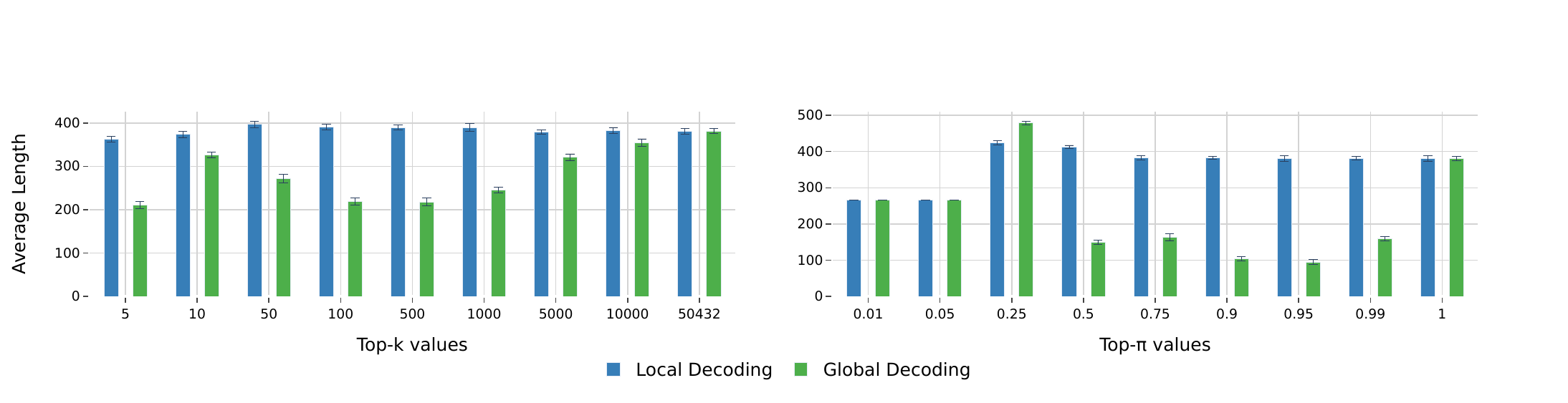}
    \vspace{-17pt}
    \caption{Average length of sequences using local and global decoding across  top-$\topkk$ (Left) and top-$\toppp$ (Right) decoding configurations. 
    The length of a sequence represents the number of tokens including $\eos$. 
    Sequences were generated with pythia-2.8b. 
    Results are averaged over 10 runs, and 95\% confidence intervals are shown.}
    \label{fig:seq_lengths}
\end{figure*}

\begin{figure*}[t]
    \centering
    \includegraphics[trim={0 .2cm 1.7cm 2cm},clip,width=\linewidth]{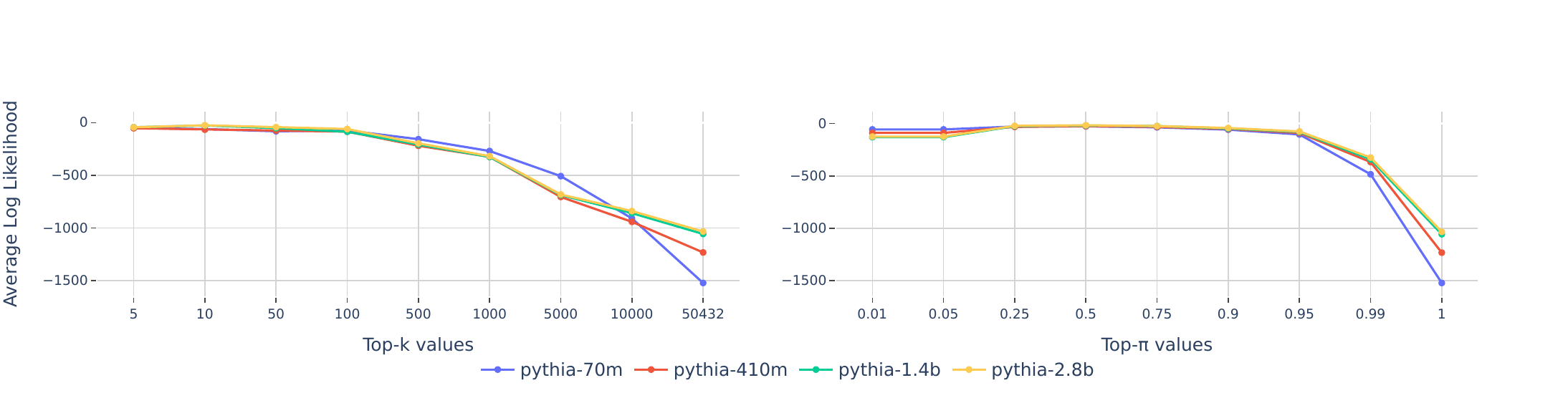}
    \vspace{-17pt}
    \caption{Log-likelihood of globally decoded samples under the original model's distribution $\ptheta$. 
    We sample sequences as: $\words\sim\decdistglobal(\words)$. 
    We then evaluate their log-probabilities as: $\log \ptheta(\words)$. Results are averaged over 10 runs.\looseness=-1}
    \label{fig:average_log_likelihood}
    \vspace{-5pt}
\end{figure*}

\begin{figure*}[t]
    \centering
    \includegraphics[trim={0 .2cm 1.7cm 2cm},clip,width=\linewidth]{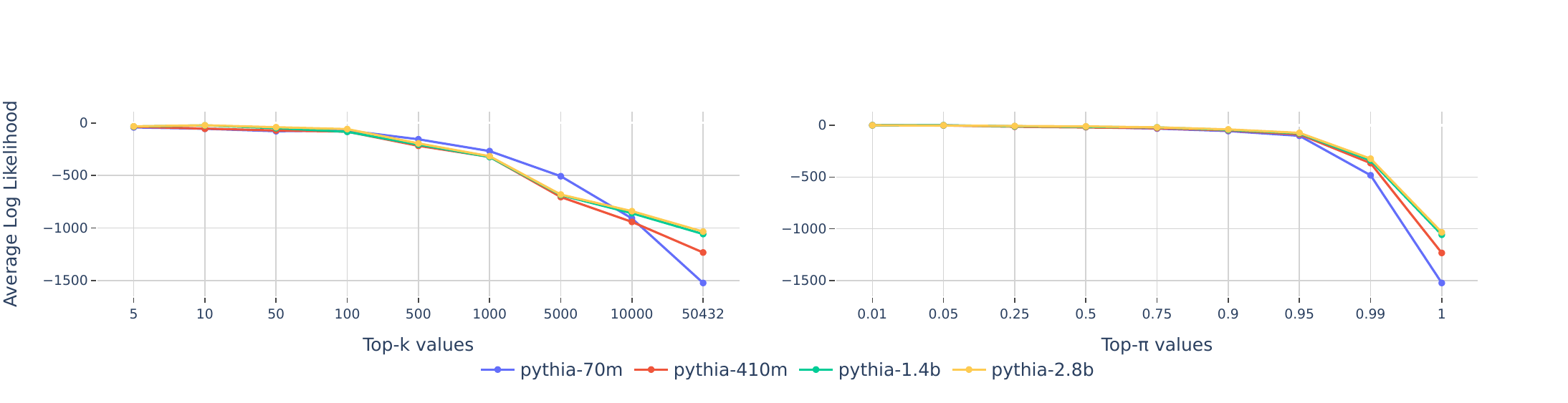}
    \vspace{-17pt}
    \caption{Log-likelihood of globally decoded samples under the local decoding distribution $\decdist$. 
    We sample sequences as: $\words\sim\decdistglobal(\words)$. 
    We then evaluate their log-probabilities as: $\log \decdist(\words)$. Results are averaged over 10 runs.\looseness=-1}
    \label{fig:average_log_likelihood_local}
\end{figure*}

\subsection{A Qualitative Analysis of Sampled Strings} 

In \cref{tab:samples} (shown in \cref{app:samples_table}), we present samples generated by global and local decoding using pythia-1.4b.
This table shows qualitative differences between the texts generated using the two methods.
In the case of local decoding, the text samples are generally coherent, with well-formed sentences and logical flow.
In contrast, global decoding often produces text with irrelevant symbols, code snippets, or disjointed phrases. 
Overall, the text generated by global decoding tends to lack the coherence and fluency observed in the local decoding samples.

\section{Discussion: Why Local Decoding Outperforms Global Decoding?}

Interestingly, despite the theoretical advantages of global decoding,  local decoding consistently outperforms it in most scenarios.
We now propose three key reasons for this discrepancy.

First, sequences generated through global decoding tend to be shorter on average (see \cref{fig:seq_lengths}).
As probabilities compound multiplicatively over the length of a sequence, and are constrained to values no greater than one, the overall probability of longer sequences decreases exponentially.
In contrast, local normalisation constants (used to normalise contextual probabilities in local decoding) are at least one. When compounded, they can thus lead to exponentially higher probabilities for longer sequences under $\decdist$.
As a result, local decoding thus tends to favour longer sequences when compared to global decoding.

Second, global decoding suffers from the well-known \defn{probability-quality paradox} in text generation \citep{holtzman2020curious,zhang-etal-2021-trading,meister-etal-2022-high}.
This paradox states that the most likely sequences according to a language model are not always the highest quality ones.
In fact, high-probability sequences frequently include repetitive or incoherent content.
Prior research has shown, for instance, that maximising likelihood during decoding, such as in beam search, often leads to degenerate and repetitive outputs \citep{holtzman2020curious}. 
In contrast, stochastic methods like top-$\toppp$ sampling produce more diverse, human-like text \citep{meister-etal-2022-high}. 
Similarly, \citet{stahlberg-byrne-2019-nmt} show that beam search often does not return the most likely string under a language model, and that these ``search errors'' improve its generated text quality.
These findings align with our results, which show that global decoding---by assigning higher probability to sequences which are already likely under the original model---tends to generate less coherent and lower-quality text.\footnote{
In \cref{fig:average_log_likelihood}, we show the log-likelihood of globally decoded sequences evaluated under the original language model's distribution $\ptheta$. 
This figure shows that for lower values of $\topkk$ or $\toppp$, the likelihoods under $\ptheta$ are higher.
\cref{fig:average_log_likelihood_local} then shows the same is true when likelihoods are evaluated using the locally normalised distribution.}

Lastly, language models often allocate significant probability mass to repetitive sequences.
Even under local decoding, small values of $\topkk$ can cause the model to get stuck in repetitive loops.
These loops typically have high local normalisation constants, as the repetitive continuations tend to be high probability and are often included top-$\topkk$ set (i.e., in $\alphabetkeeptopk$). 
In such cases, local decoding effectively discounts these sequences by adjusting their probabilities.
In contrast, global decoding does not apply this discount, giving these repetitive sequences a higher chance of being sampled. 
This results in less coherent and more repetitive text under global normalisation.

\section{Conclusion}

In this paper, we explored the effects of local and global normalisation on text generation quality, focusing on two popular decoding methods: top-$\topkk$ and top-$\toppp$.
We introduced the concept of global decoding algorithms and adapted local decoding algorithms to their global counterparts. 
Our empirical comparison revealed that, while global decoding preserves the original distribution, it often produces shorter, less coherent and more repetitive text than local decoding.
In contrast, local decoding---despite introducing distortions to the original distribution---typically results in more coherent text.
Our findings thus suggest that the distortion introduced by local decoding algorithms might be a beneficial feature rather than a flaw.
Future research on decoding algorithms should explore and evaluate the effects of local normalisation alongside pruning strategies.

\section*{Limitations}

As with any research project, our work has theoretical and empirical limitations.
Theoretically, the lower bounds presented in \cref{sec:local_global_decoding} are asymptotic and omit key constants required to fully understand how different decoding algorithms and their hyperparameters influence performance.
Additionally, we cannot directly sample from the global decoding distributions and rely on IMH for approximate sampling.
Although we examine the impact of this choice in our experiments (see \cref{fig:mcmc_iterations_plot}), it may still affect the results.
Empirically, our experiments are limited to Pythia models ranging in size from 70m to 2.8b.
Due to large number of samples required by our sampling methods and relatively large number of hyperparameter configurations analysed, we were limited to using smaller-sized models.
As models continue to improve, global distributions may become better calibrated, potentially yielding stronger results.
Additionally, we only run experiments in English. 
Extending this analysis to other languages is an important next step.
Furthermore, our focus was limited to open-ended text generation and applying our approach to other natural language tasks, such as summarisation, could provide further insights.
Finally, our evaluation primarily relies on automatic metrics such as MAUVE, and we do not perform human evaluations.
Although MAUVE scores have been shown to correlate well with human judgements \citep{pillutla-etal:mauve:neurips2021}, they are not a definitive standard. 
Human evaluations would offer a more comprehensive understanding of the results. 
Moreover, incorporating other automatic metrics, such as Zipf Coefficient \citep{holtzman2020curious} and Generation Perplexity \citep{fan-etal-2018-hierarchical}, could provide a more well-rounded assessment of decoding performance.

\section*{Acknowledgement}

We thank Pietro Lesci, Gregor Bachmann, Yahya Emara, Clara Meister, Afra Amini and Sotiris Anagnostidis for their feedback on earlier versions of this paper.
We also thank our anonymous reviewers.

\bibliography{custom}

\begin{thebibliography}{23}
\expandafter\ifx\csname natexlab\endcsname\relax\def\natexlab#1{#1}\fi

\bibitem[{Amini et~al.(2023)Amini, Du, and Cotterell}]{amini2023structured}
Afra Amini, Li~Du, and Ryan Cotterell. 2023.
\newblock \href {https://openreview.net/forum?id=vf77fTbgG3} {Structured
  voronoi sampling}.
\newblock In \emph{Thirty-seventh Conference on Neural Information Processing
  Systems}.

\bibitem[{Basu et~al.(2021)Basu, Ramachandran, Keskar, and
  Varshney}]{basu2021mirostat}
Sourya Basu, Govardana~Sachitanandam Ramachandran, Nitish~Shirish Keskar, and
  Lav~R. Varshney. 2021.
\newblock \href {https://openreview.net/forum?id=W1G1JZEIy5_} {Mirostat: A
  neural text decoding algorithm that directly controls perplexity}.
\newblock In \emph{International Conference on Learning Representations}.

\bibitem[{Biderman et~al.(2023)Biderman, Schoelkopf, Anthony, Bradley, O'Brien,
  Hallahan, Khan, Purohit, Prashanth, Raff, Skowron, Sutawika, and Van
  Der~Wal}]{biderman2023pythia}
Stella Biderman, Hailey Schoelkopf, Quentin Anthony, Herbie Bradley, Kyle
  O'Brien, Eric Hallahan, Mohammad~Aflah Khan, Shivanshu Purohit, USVSN~Sai
  Prashanth, Edward Raff, Aviya Skowron, Lintang Sutawika, and Oskar Van
  Der~Wal. 2023.
\newblock \href {https://dl.acm.org/doi/10.5555/3618408.3618510} {Pythia: a
  suite for analyzing large language models across training and scaling}.
\newblock In \emph{Proceedings of the 40th International Conference on Machine
  Learning}, ICML'23. JMLR.org.

\bibitem[{Du et~al.(2024)Du, Amini, Hennigen, Yu, Lee, Eisner, and
  Cotterell}]{du2023principled}
Li~Du, Afra Amini, Lucas~Torroba Hennigen, Xinyan~Velocity Yu, Holden Lee,
  Jason Eisner, and Ryan Cotterell. 2024.
\newblock \href {https://openreview.net/forum?id=AwLLSlJAeJ} {Principled
  gradient-based {MCMC} for conditional sampling of text}.
\newblock In \emph{Forty-first International Conference on Machine Learning}.

\bibitem[{Fan et~al.(2018)Fan, Lewis, and Dauphin}]{fan-etal-2018-hierarchical}
Angela Fan, Mike Lewis, and Yann Dauphin. 2018.
\newblock \href {https://doi.org/10.18653/v1/P18-1082} {Hierarchical neural
  story generation}.
\newblock In \emph{Proceedings of the 56th Annual Meeting of the Association
  for Computational Linguistics (Volume 1: Long Papers)}, pages 889--898,
  Melbourne, Australia. Association for Computational Linguistics.

\bibitem[{Forristal et~al.(2023)Forristal, Mireshghallah, Durrett, and
  Berg-Kirkpatrick}]{forristal-etal-2023-block}
Jarad Forristal, Fatemehsadat Mireshghallah, Greg Durrett, and Taylor
  Berg-Kirkpatrick. 2023.
\newblock \href {https://aclanthology.org/2023.conll-1.26} {A block
  {M}etropolis-{H}astings sampler for controllable energy-based text
  generation}.
\newblock In \emph{Proceedings of the 27th Conference on Computational Natural
  Language Learning (CoNLL)}, pages 403--413, Singapore. Association for
  Computational Linguistics.

\bibitem[{Goyal et~al.(2019)Goyal, Dyer, and
  Berg-Kirkpatrick}]{goyal-etal-2019-empirical}
Kartik Goyal, Chris Dyer, and Taylor Berg-Kirkpatrick. 2019.
\newblock \href {https://doi.org/10.18653/v1/N19-1171} {{A}n empirical
  investigation of global and local normalization for recurrent neural sequence
  models using a continuous relaxation to beam search}.
\newblock In \emph{Proceedings of the 2019 Conference of the North {A}merican
  Chapter of the Association for Computational Linguistics: Human Language
  Technologies, Volume 1 (Long and Short Papers)}, pages 1724--1733,
  Minneapolis, Minnesota. Association for Computational Linguistics.

\bibitem[{Hewitt et~al.(2022)Hewitt, Manning, and
  Liang}]{hewitt-etal-2022-truncation}
John Hewitt, Christopher Manning, and Percy Liang. 2022.
\newblock \href {https://doi.org/10.18653/v1/2022.findings-emnlp.249}
  {Truncation sampling as language model desmoothing}.
\newblock In \emph{Findings of the Association for Computational Linguistics:
  EMNLP 2022}, pages 3414--3427, Abu Dhabi, United Arab Emirates. Association
  for Computational Linguistics.

\bibitem[{Holtzman et~al.(2020)Holtzman, Buys, Du, Forbes, and
  Choi}]{holtzman2020curious}
Ari Holtzman, Jan Buys, Li~Du, Maxwell Forbes, and Yejin Choi. 2020.
\newblock \href {https://openreview.net/forum?id=rygGQyrFvH} {The curious case
  of neural text degeneration}.
\newblock In \emph{International Conference on Learning Representations}.

\bibitem[{Lew et~al.(2023)Lew, Zhi-Xuan, Grand, and
  Mansinghka}]{lew2023sequential}
Alexander~K. Lew, Tan Zhi-Xuan, Gabriel Grand, and Vikash~K. Mansinghka. 2023.
\newblock \href {http://arxiv.org/abs/2306.03081} {Sequential {M}onte {C}arlo
  steering of large language models using probabilistic programs}.
\newblock \emph{arXiv preprint}.

\bibitem[{Meister et~al.(2023{\natexlab{a}})Meister, Pimentel, Malagutti,
  Wilcox, and Cotterell}]{meister-etal-2023-efficacy}
Clara Meister, Tiago Pimentel, Luca Malagutti, Ethan Wilcox, and Ryan
  Cotterell. 2023{\natexlab{a}}.
\newblock \href {https://doi.org/10.18653/v1/2023.acl-long.80} {On the efficacy
  of sampling adapters}.
\newblock In \emph{Proceedings of the 61st Annual Meeting of the Association
  for Computational Linguistics (Volume 1: Long Papers)}, pages 1437--1455,
  Toronto, Canada. Association for Computational Linguistics.

\bibitem[{Meister et~al.(2023{\natexlab{b}})Meister, Pimentel, Wiher, and
  Cotterell}]{meister-etal-2023-locally}
Clara Meister, Tiago Pimentel, Gian Wiher, and Ryan Cotterell.
  2023{\natexlab{b}}.
\newblock \href {https://doi.org/10.1162/tacl_a_00536} {Locally typical
  sampling}.
\newblock \emph{Transactions of the Association for Computational Linguistics},
  11:102--121.

\bibitem[{Meister et~al.(2022)Meister, Wiher, Pimentel, and
  Cotterell}]{meister-etal-2022-high}
Clara Meister, Gian Wiher, Tiago Pimentel, and Ryan Cotterell. 2022.
\newblock \href {https://doi.org/10.18653/v1/2022.acl-short.5} {On the
  probability{--}quality paradox in language generation}.
\newblock In \emph{Proceedings of the 60th Annual Meeting of the Association
  for Computational Linguistics (Volume 2: Short Papers)}, pages 36--45,
  Dublin, Ireland. Association for Computational Linguistics.

\bibitem[{Miao et~al.(2019)Miao, Zhou, Mou, Yan, and Li}]{miao2019cgmh}
Ning Miao, Hao Zhou, Lili Mou, Rui Yan, and Lei Li. 2019.
\newblock \href {https://cdn.aaai.org/ojs/4659/4659-13-7698-1-10-20190707.pdf}
  {{CGMH}: Constrained sentence generation by {M}etropolis-{H}astings
  sampling}.
\newblock In \emph{The Thirty-Third AAAI Conference on Artificial Intelligence
  (AAAI-19)}.

\bibitem[{OpenAI et~al.(2024)OpenAI, Achiam, Adler, Agarwal, Ahmad, Akkaya,
  Aleman, Almeida, Altenschmidt, Altman, Anadkat, Avila, Babuschkin, Balaji,
  Balcom, Baltescu, Bao, Bavarian, Belgum, Bello et~al.}]{achiam2023gpt}
OpenAI, Josh Achiam, Steven Adler, Sandhini Agarwal, Lama Ahmad, Ilge Akkaya,
  Florencia~Leoni Aleman, Diogo Almeida, Janko Altenschmidt, Sam Altman,
  Shyamal Anadkat, Red Avila, Igor Babuschkin, Suchir Balaji, Valerie Balcom,
  Paul Baltescu, Haiming Bao, Mohammad Bavarian, Jeff Belgum, Irwan Bello,
  et~al. 2024.
\newblock \href {http://arxiv.org/abs/2303.08774} {{GPT}-4 technical report}.
\newblock \emph{arXiv preprint}.

\bibitem[{Pillutla et~al.(2021)Pillutla, Swayamdipta, Zellers, Thickstun,
  Welleck, Choi, and Harchaoui}]{pillutla-etal:mauve:neurips2021}
Krishna Pillutla, Swabha Swayamdipta, Rowan Zellers, John Thickstun, Sean
  Welleck, Yejin Choi, and Zaid Harchaoui. 2021.
\newblock \href {https://openreview.net/forum?id=Tqx7nJp7PR} {{MAUVE}:
  Measuring the gap between neural text and human text using divergence
  frontiers}.
\newblock In \emph{Advances in Neural Information Processing Systems}.

\bibitem[{Radford et~al.(2019)Radford, Wu, Child, Luan, Amodei, and
  Sutskever}]{radford2019language}
Alec Radford, Jeffrey Wu, Rewon Child, David Luan, Dario Amodei, and Ilya
  Sutskever. 2019.
\newblock \href
  {https://cdn.openai.com/better-language-models/language_models_are_unsupervised_multitask_learners.pdf}
  {Language models are unsupervised multitask learners}.
\newblock \emph{OpenAI blog}, 1(8):9.

\bibitem[{Stahlberg and Byrne(2019)}]{stahlberg-byrne-2019-nmt}
Felix Stahlberg and Bill Byrne. 2019.
\newblock \href {https://doi.org/10.18653/v1/D19-1331} {On {NMT} search errors
  and model errors: Cat got your tongue?}
\newblock In \emph{Proceedings of the 2019 Conference on Empirical Methods in
  Natural Language Processing and the 9th International Joint Conference on
  Natural Language Processing (EMNLP-IJCNLP)}, pages 3356--3362, Hong Kong,
  China. Association for Computational Linguistics.

\bibitem[{Wang(2021)}]{wang2021exact}
Guanyang Wang. 2021.
\newblock \href {http://arxiv.org/abs/2008.02455} {Exact convergence rate
  analysis of the independent {M}etropolis--{H}astings algorithms}.
\newblock \emph{arXiv preprint}.

\bibitem[{Welleck et~al.(2024)Welleck, Bertsch, Finlayson, Schoelkopf, Xie,
  Neubig, Kulikov, and Harchaoui}]{welleck2024decoding}
Sean Welleck, Amanda Bertsch, Matthew Finlayson, Hailey Schoelkopf, Alex Xie,
  Graham Neubig, Ilia Kulikov, and Zaid Harchaoui. 2024.
\newblock \href {http://arxiv.org/abs/2406.16838} {From decoding to
  meta-generation: Inference-time algorithms for large language models}.
\newblock \emph{arXiv preprint}.

\bibitem[{Yang and Klein(2021)}]{yang-klein-2021-fudge}
Kevin Yang and Dan Klein. 2021.
\newblock \href {https://doi.org/10.18653/v1/2021.naacl-main.276} {{FUDGE}:
  Controlled text generation with future discriminators}.
\newblock In \emph{Proceedings of the 2021 Conference of the North American
  Chapter of the Association for Computational Linguistics: Human Language
  Technologies}, pages 3511--3535, Online. Association for Computational
  Linguistics.

\bibitem[{Zhang et~al.(2021)Zhang, Duckworth, Ippolito, and
  Neelakantan}]{zhang-etal-2021-trading}
Hugh Zhang, Daniel Duckworth, Daphne Ippolito, and Arvind Neelakantan. 2021.
\newblock \href {https://aclanthology.org/2021.humeval-1.3} {Trading off
  diversity and quality in natural language generation}.
\newblock In \emph{Proceedings of the Workshop on Human Evaluation of NLP
  Systems (HumEval)}, pages 25--33, Online. Association for Computational
  Linguistics.

\bibitem[{Zhu et~al.(2018)Zhu, Lu, Zheng, Guo, Zhang, Wang, and
  Yu}]{zhu2018texygen}
Yaoming Zhu, Sidi Lu, Lei Zheng, Jiaxian Guo, Weinan Zhang, Jun Wang, and Yong
  Yu. 2018.
\newblock \href {https://doi.org/10.1145/3209978.3210080} {Texygen: A
  benchmarking platform for text generation models}.
\newblock In \emph{The 41st International ACM SIGIR Conference on Research \&
  Development in Information Retrieval}, SIGIR '18, page 1097–1100, New York,
  NY, USA. Association for Computing Machinery.

\end{thebibliography}

\onecolumn
\appendix

\section{Detailed Results} \label{app:detail_results}

\subsection{MAUVE Scores}

\begin{table*}[h!]
\centering
\resizebox{\textwidth}{!}{
\begin{tabular}{lcccccccl}
\toprule
& 
\multicolumn{2}{c}{pythia-70m}
& 
\multicolumn{2}{c}{pythia-410m}
& 
\multicolumn{2}{c}{pythia-1.4b}
& 
\multicolumn{2}{c}{pythia-2.8b} \\
\cmidrule(lr){2-3}\cmidrule(lr){4-5}
\cmidrule(lr){6-7}\cmidrule(lr){8-9}
           & $\mauvelocal$  & $\mauveglobal$ & $\mauvelocal$  & $\mauveglobal$ & $\mauvelocal$  & $\mauveglobal$ & $\mauvelocal$  & $\mauveglobal$ \\\midrule
$\alphabetkeeptopkwithval{|\alphabet|}$  & 0.484 ± 0.075 & 0.473 ± 0.152 & 0.593 ± 0.085 & 0.634 ± 0.062 & 0.612 ± 0.085 & 0.526 ± 0.121 & 0.587 ± 0.128 & 0.534 ± 0.072 \\ \midrule
$\alphabetkeeptopkwithval{5.0}$     & 0.058 ± 0.018 & 0.013 ± 0.004 & 0.091 ± 0.036 & 0.017 ± 0.007 & 0.074 ± 0.037 & 0.017 ± 0.008 & 0.075 ± 0.024 & 0.048 ± 0.011 \\ 
$\alphabetkeeptopkwithval{10.0}$     & 0.082 ± 0.034 & 0.019 ± 0.002 & 0.157 ± 0.059 & 0.073 ± 0.033 & 0.096 ± 0.037 & 0.018 ± 0.006 & 0.084 ± 0.029 & 0.014 ± 0.004 \\
$\alphabetkeeptopkwithval{50.0}$     & 0.299 ± 0.078 & 0.041 ± 0.009 & 0.374 ± 0.088 & 0.047 ± 0.018 & 0.323 ± 0.073 & 0.039 ± 0.012 & 0.288 ± 0.064 & 0.038 ± 0.018 \\ 
$\alphabetkeeptopkwithval{100.0}$     & 0.374 ± 0.107 & 0.049 ± 0.017 & 0.467 ± 0.103 & 0.047 ± 0.011 & 0.442 ± 0.109 & 0.052 ± 0.017 & 0.377 ± 0.163 & 0.061 ± 0.017 \\ 
$\alphabetkeeptopkwithval{500.0}$     & 0.453 ± 0.131 & 0.079 ± 0.03 & 0.546 ± 0.078 & 0.112 ± 0.043 & 0.501 ± 0.061 & 0.098 ± 0.03 & 0.474 ± 0.097 & 0.105 ± 0.04 \\
$\alphabetkeeptopkwithval{1000.0}$     & 0.506 ± 0.128 & 0.1 ± 0.035 & 0.605 ± 0.068 & 0.149 ± 0.06 & 0.566 ± 0.128 & 0.143 ± 0.058 & 0.489 ± 0.061 & 0.154 ± 0.07 \\ 
$\alphabetkeeptopkwithval{5000.0}$     & 0.495 ± 0.109 & 0.194 ± 0.068 & 0.652 ± 0.098 & 0.359 ± 0.13 & 0.604 ± 0.111 & 0.379 ± 0.113 & 0.539 ± 0.121 & 0.374 ± 0.121 \\ 
$\alphabetkeeptopkwithval{10000.0}$     & 0.538 ± 0.14 & 0.299 ± 0.051 & 0.6 ± 0.102 & 0.53 ± 0.111 & 0.582 ± 0.138 & 0.427 ± 0.085 & 0.515 ± 0.128 & 0.431 ± 0.096 \\ \midrule

$\alphabetkeeptoppwithval{0.01}$     & 0.004 ± 0.0 & 0.004 ± 0.0 & 0.004 ± 0.0 & 0.004 ± 0.0 & 0.004 ± 0.0 & 0.004 ± 0.0 & 0.004 ± 0.0 & 0.004 ± 0.0 \\
$\alphabetkeeptoppwithval{0.05}$     & 0.004 ± 0.001 & 0.004 ± 0.0 & 0.005 ± 0.001 & 0.004 ± 0.0 & 0.004 ± 0.0 & 0.004 ± 0.0 & 0.004 ± 0.0 & 0.004 ± 0.0 \\ 
$\alphabetkeeptoppwithval{0.25}$     & 0.015 ± 0.004 & 0.011 ± 0.005 & 0.026 ± 0.008 & 0.014 ± 0.006 & 0.032 ± 0.017 & 0.007 ± 0.002 & 0.03 ± 0.02 & 0.009 ± 0.001 \\
$\alphabetkeeptoppwithval{0.5}$     & 0.041 ± 0.014 & 0.019 ± 0.005 & 0.112 ± 0.039 & 0.034 ± 0.015 & 0.14 ± 0.066 & 0.022 ± 0.007 & 0.138 ± 0.041 & 0.014 ± 0.004 \\
$\alphabetkeeptoppwithval{0.75}$     & 0.283 ± 0.054 & 0.101 ± 0.029 & 0.478 ± 0.146 & 0.077 ± 0.023 & 0.415 ± 0.066 & 0.033 ± 0.009 & 0.411 ± 0.09 & 0.067 ± 0.023 \\ 
$\alphabetkeeptoppwithval{0.9}$     & 0.464 ± 0.108 & 0.157 ± 0.055 & 0.566 ± 0.078 & 0.113 ± 0.04 & 0.566 ± 0.132 & 0.076 ± 0.015 & 0.507 ± 0.099 & 0.11 ± 0.037 \\ 
$\alphabetkeeptoppwithval{0.95}$     & 0.486 ± 0.138 & 0.19 ± 0.048 & 0.623 ± 0.073 & 0.186 ± 0.056 & 0.575 ± 0.135 & 0.124 ± 0.041 & 0.489 ± 0.086 & 0.169 ± 0.053 \\ 
$\alphabetkeeptoppwithval{0.99}$     & 0.484 ± 0.107 & 0.232 ± 0.077 & 0.595 ± 0.091 & 0.328 ± 0.138 & 0.621 ± 0.105 & 0.291 ± 0.091 & 0.561 ± 0.132 & 0.278 ± 0.095 \\
$\alphabetkeeptoppwithval{1.0}$     & 0.484 ± 0.075 & 0.473 ± 0.152 & 0.593 ± 0.085 & 0.634 ± 0.062 & 0.612 ± 0.085 & 0.526 ± 0.121 & 0.587 ± 0.128 & 0.534 ± 0.072 
\\\bottomrule
\end{tabular}
}
  \caption{MAUVE evaluation scores when using local and global decoding with various top-$\topkk$ and top-$\toppp$ settings. 
  Results are averaged over 10 runs, and 95\% confidence intervals are shown.}
\label{table:mauve}
\end{table*}

\subsection{Self-BLEU Scores}

\begin{table*}[h!]
\centering
\resizebox{\textwidth}{!}{%
\begin{tabular}{lcccccccl}
\toprule
& 
\multicolumn{2}{c}{pythia-70m}
& 
\multicolumn{2}{c}{pythia-410m}
& 
\multicolumn{2}{c}{pythia-1.4b}
& 
\multicolumn{2}{c}{pythia-2.8b} \\
\cmidrule(lr){2-3}\cmidrule(lr){4-5}
\cmidrule(lr){6-7}\cmidrule(lr){8-9}
& $\bleulocal$  & $\bleuglobal$ & $\bleulocal$  & $\bleuglobal$ & $\bleulocal$  & $\bleuglobal$ & $\bleulocal$  & $\bleuglobal$ \\\midrule
$\alphabetkeeptopkwithval{|\alphabet|}$    & 0.0006 ± 0.0001 & 0.0005 ± 0.0002 & 0.0005 ± 0.0002 & 0.0005 ± 0.0001 & 0.0005 ± 0.0001 & 0.0005 ± 0.0002 & 0.0005 ± 0.0001 & 0.0005 ± 0.0002 \\ \midrule
    $\alphabetkeeptopkwithval{5.0}$     & 0.005 ± 0.0007 & 0.0164 ± 0.0014 & 0.0069 ± 0.0007 & 0.0206 ± 0.0021 & 0.0066 ± 0.0008 & 0.0765 ± 0.0101 & 0.0074 ± 0.0007 & 0.0325 ± 0.0056 \\ 
$\alphabetkeeptopkwithval{10.0}$     & 0.0037 ± 0.0007 & 0.0182 ± 0.0041 & 0.0044 ± 0.0008 & 0.0079 ± 0.0013 & 0.0044 ± 0.0003 & 0.0712 ± 0.0095 & 0.0047 ± 0.0004 & 0.1365 ± 0.0145 \\ 
$\alphabetkeeptopkwithval{50.0}$     & 0.0017 ± 0.0004 & 0.0066 ± 0.0032 & 0.0016 ± 0.0003 & 0.0102 ± 0.002 & 0.0018 ± 0.0004 & 0.0132 ± 0.0014 & 0.0018 ± 0.0003 & 0.033 ± 0.005 \\ 
$\alphabetkeeptopkwithval{100.0}$     & 0.0014 ± 0.0002 & 0.0109 ± 0.0022 & 0.0012 ± 0.0003 & 0.0124 ± 0.0026 & 0.0014 ± 0.0004 & 0.0084 ± 0.0022 & 0.0015 ± 0.0003 & 0.0137 ± 0.0021 \\ 
$\alphabetkeeptopkwithval{500.0}$     & 0.0008 ± 0.0001 & 0.0044 ± 0.0008 & 0.0008 ± 0.0002 & 0.0045 ± 0.0008 & 0.0008 ± 0.0001 & 0.0038 ± 0.0009 & 0.0009 ± 0.0002 & 0.0032 ± 0.0008 \\ 
$\alphabetkeeptopkwithval{1000.0}$     & 0.0008 ± 0.0002 & 0.0042 ± 0.001 & 0.0006 ± 0.0002 & 0.0033 ± 0.001 & 0.0007 ± 0.0001 & 0.0024 ± 0.0007 & 0.0008 ± 0.0002 & 0.0017 ± 0.0005 \\ 
$\alphabetkeeptopkwithval{5000.0}$     & 0.0006 ± 0.0001 & 0.0008 ± 0.0003 & 0.0005 ± 0.0002 & 0.0006 ± 0.0002 & 0.0005 ± 0.0001 & 0.0007 ± 0.0002 & 0.0005 ± 0.0001 & 0.0007 ± 0.0002 \\ 
$\alphabetkeeptopkwithval{10000.0}$     & 0.0006 ± 0.0001 & 0.0006 ± 0.0001 & 0.0005 ± 0.0002 & 0.0006 ± 0.0001 & 0.0005 ± 0.0002 & 0.0006 ± 0.0002 & 0.0006 ± 0.0002 & 0.0006 ± 0.0002 \\ \midrule

$\alphabetkeeptoppwithval{0.01}$     & 1.0 ± 0.0 & 1.0 ± 0.0 & 1.0 ± 0.0 & 1.0 ± 0.0 & 1.0 ± 0.0 & 1.0 ± 0.0 & 1.0 ± 0.0 & 1.0 ± 0.0 \\ 
$\alphabetkeeptoppwithval{0.05}$     & 0.3584 ± 0.0104 & 1.0 ± 0.0 & 0.4466 ± 0.0128 & 1.0 ± 0.0 & 1.0 ± 0.0 & 1.0 ± 0.0 & 1.0 ± 0.0 & 1.0 ± 0.0 \\ 
$\alphabetkeeptoppwithval{0.25}$     & 0.0135 ± 0.0024 & 0.0234 ± 0.0045 & 0.0035 ± 0.0006 & 0.0268 ± 0.0058 & 0.0056 ± 0.001 & 0.2687 ± 0.0218 & 0.0105 ± 0.0018 & 0.5187 ± 0.0204 \\ 
$\alphabetkeeptoppwithval{0.5}$     & 0.002 ± 0.0004 & 0.0165 ± 0.003 & 0.0014 ± 0.0003 & 0.0306 ± 0.0048 & 0.0016 ± 0.0003 & 0.0534 ± 0.0116 & 0.0021 ± 0.0003 & 0.1093 ± 0.0228 \\ 
$\alphabetkeeptoppwithval{0.75}$     & 0.0016 ± 0.0002 & 0.0132 ± 0.0041 & 0.0012 ± 0.0002 & 0.0103 ± 0.003 & 0.0012 ± 0.0001 & 0.0123 ± 0.0018 & 0.0013 ± 0.0002 & 0.0188 ± 0.004 \\ 
$\alphabetkeeptoppwithval{0.9}$     & 0.001 ± 0.0001 & 0.0019 ± 0.0004 & 0.0007 ± 0.0001 & 0.0026 ± 0.001 & 0.0009 ± 0.0001 & 0.0038 ± 0.0008 & 0.0008 ± 0.0001 & 0.008 ± 0.0014 \\ 
$\alphabetkeeptoppwithval{0.95}$     & 0.0007 ± 0.0001 & 0.0006 ± 0.0002 & 0.0006 ± 0.0002 & 0.0011 ± 0.0005 & 0.0006 ± 0.0001 & 0.0017 ± 0.0005 & 0.0008 ± 0.0002 & 0.004 ± 0.0015 \\ 
$\alphabetkeeptoppwithval{0.99}$     & 0.0006 ± 0.0001 & 0.0002 ± 0.0001 & 0.0005 ± 0.0003 & 0.0003 ± 0.0001 & 0.0005 ± 0.0001 & 0.0005 ± 0.0002 & 0.0006 ± 0.0001 & 0.0005 ± 0.0001 \\ 
$\alphabetkeeptoppwithval{1.0}$     & 0.0006 ± 0.0001 & 0.0005 ± 0.0002 & 0.0005 ± 0.0002 & 0.0005 ± 0.0001 & 0.0005 ± 0.0001 & 0.0005 ± 0.0002 & 0.0005 ± 0.0001 & 0.0005 ± 0.0002 

\\\bottomrule
\end{tabular}
}
\vspace{-2pt}
\caption{Self-BLEU evaluation scores when using local and global decoding with various top-$\topkk$ and top-$\toppp$ settings. 
  Results are averaged over 10 runs, and 95\% confidence intervals are shown.}
\label{table:bleu}
\vspace{-3pt}
\end{table*}

\clearpage

\subsection{Samples of Strings} \label{app:samples_table}

\newcommand{\tablelinespacing}{7pt}
\begin{table*}[h!]
    \centering
    \label{crouch}
    \resizebox{\textwidth}{!}{%
    \begin{tabular}{  l  p{8cm}  p{8cm} }
        \toprule
& \multicolumn{1}{c}{\textbf{Local Decoding}}
& \multicolumn{1}{c}{\textbf{Global Decoding}} \\\midrule
$\alphabetkeeptopkwithval{|\alphabet|}$  
&        
\detokenize{
Xmas is exciting this year with the discovery of the new Dave.\n\nBelow you can see the new Dave in a 
}
& 
\detokenize{
Q:\n\nHow to start VLC using jars from android\n\nI am trying to start VLC using JPMS provided by Wizard
}
\\\midrule
$\alphabetkeeptopkwithval{5}$  
&        
\detokenize{
Q:\n\nHow to use the new.NET 4.0 framework with Mono 2.8.7 on Ubuntu 11.10?\n\nHow can I use the new.NET
}
& 
\detokenize{
\n#ifndef BOOST_MPL_AUX_TEMPLATE_ARIT
Y_HPP_INCLUDED\n#define BOOST_MPL_AUX_TEMPLATE_ARITY_HPP_IN
CLUDED
}
\\[\tablelinespacing]
$\alphabetkeeptopkwithval{10}$ 
&        
\detokenize{
\n\nI just read about the first time the UW football team beat Iowa, and I thought it was really cool:
}
& 
\detokenize{
1\n
}
\\[\tablelinespacing]
$\alphabetkeeptopkwithval{50}$  
&        
\detokenize{
Hurricane Dorian has left at least 15 dead and hundreds missing in the Bahamas as it approaches the
}
& 
\detokenize{
#ifndef OSQA_CONFIG_H_\n#define OSQA_CONFIG_H_\n\n/** @file\n *\n * Definitions shared by all platforms\n 
}
\\[\tablelinespacing]
$\alphabetkeeptopkwithval{100}$  
&        
\detokenize{
Tag: politics\n\nOn the afternoon of Wednesday, May 17th, I joined hundreds of thousands of other Amer
}
& 
\detokenize{
/*\nCopyright The Kubernetes Authors.\n\nLicensed under the Apache License, Version 2.0 (the "License")
}
\\[\tablelinespacing]
$\alphabetkeeptopkwithval{500}$  
&        
\detokenize{
Results of hyperthermic intraperitoneal chemotherapy versus hyperthermic intraperitoneal chemotherap
}
& 
\detokenize{
<?xml version="1.0" encoding="UTF-8"?>\n<!--  Copyright (C) 2013 The Android Open Source Project\n\n
}
\\[\tablelinespacing]
$\alphabetkeeptopkwithval{1000}$   
&        
\detokenize{
Peace love, Hope and Fasting.\n\nThursday, November 10, 2008\n\nOkay...A few weeks ago I had an interest
}
& 
\detokenize{
define("ace/mode/javascript_highlight_rules",["requ
ire","exports","module","ace/lib/oop","ace/mode/t
}
\\[\tablelinespacing]
$\alphabetkeeptopkwithval{5000}$  
&        
\detokenize{
West Lafayette\n\nIndiana’s unexpected first visit to baseball crowncases its home team’s success from
}
& 
\detokenize{
\n624 P.2d 352 (1981)\n105 Idaho 1002\nSue HANKS, Plaintiff-Appellant,\nv.\nHARTSDALE TOWNSHIP, Defendant
}
\\[\tablelinespacing]
$\alphabetkeeptopkwithval{10000}$
&        
\detokenize{
These algorithms have a "signature capture score"\nof about 3.\n
}
& 
\detokenize{
Beth Cooper\n\nNew York Times best-selling author Beth Cooper will deliver a lecture on poetry and cul
}
\\\midrule
$\alphabetkeeptoppwithval{0.01}$
&        
\detokenize{
Q:\n\nHow to get the value of a variable in a function?\n\nI have a function that is called from a butto
}
& 
\detokenize{
Q:\n\nHow to get the value of a variable in a function?\n\nI have a function that is called from a butto
}
\\[\tablelinespacing]
$\alphabetkeeptoppwithval{0.05}$
&        
\detokenize{
This invention relates to a semiconductor device and a method of manufacturing the same, and more pa
}
& 
\detokenize{
\n#include "qabstractnetworkmodel.h"\n
}
\\[\tablelinespacing]
$\alphabetkeeptoppwithval{0.25}$
&        
\detokenize{
The present invention relates to a semiconductor device and a method of manufacturing the same, and 
}
& 
\detokenize{
/*\n * Copyright (c) 2017-2018 THL A29 Limited, a Tencent company. All Rights Reserved.\n *\n * License
}
\\[\tablelinespacing]
$\alphabetkeeptoppwithval{0.5}$
&        
\detokenize{
K. T. P. S. B. P. C.\n\nK. T. P. S. B. P. C. is a small town in the Indian state of Uttar Pradesh. It 
}
& 
\detokenize{
import { IconDefinition, IconPrefix, IconName } from "@fortawesome/fontawesome-common-types";\nexport
}
\\[\tablelinespacing]
$\alphabetkeeptoppwithval{0.75}$
&        
\detokenize{
YAML 1.1\n\%TAG!u! tag:unity3d.com,2011:\n---!u!21 &2100000\nMaterial:\n  serializedVersion: 6\n  m_Objec
}
& 
\detokenize{
export { default as Slider } from './Slider';\n
}
\\[\tablelinespacing]
$\alphabetkeeptoppwithval{0.9}$
&        
\detokenize{
---\nabstract: 'We investigate the concept of the infinite divisibility of simple functions using som
}
& 
\detokenize{
Embodiments described herein relate to a device for testing nerve stimulation therapy for chronic pa
}
\\[\tablelinespacing]
$\alphabetkeeptopkwithval{0.95}$
&        
\detokenize{
Secretary of State Rex Tillerson, Defense Secretary Jim Mattis, CENTCOM Commander Army Gen. Lloyd Au
}
& 
\detokenize{
export * from './domain.module';\nexport { useLinkTo, registerLinkTo } from './utils';\n
}
\\[\tablelinespacing]
$\alphabetkeeptoppwithval{0.99}$
&        
\detokenize{
Associated Press\n\nSYRACUSE, N.Y. (AP) - A retired Northeastern University professor has been ordered
}
& 
\detokenize{
Along the dark, shadowy streets of ChicoMira moving south on Angioteque, cops have spotted three bro
}
\\
\bottomrule
\end{tabular}%
}
\caption{Texts generated with pythia-1.4b using local and global decoding. We present one randomly selected text for each configuration. The sequences are trimmed to 100 characters for readability.}
\label{tab:samples}
\end{table*}

\newpage

\section{Definitions and Useful Lemma}

For mathematical convenience, we now define local and global normalisation constants.
\begin{defin}\label{defn:local_constant}
    The \defn{local normalisation constant} (denoted $\constlocalnorm: \alphabet^{*}\to[0,1]$) is defined as:
    \begin{align}\label{eq:local_constant}
        &\constlocalnorm(\words) \defeq \sum_{\word \in \alphabeteos} \decdistunormed(\word \mid \words) = \decdistunormed(\eos \mid \words) + \sum_{\word \in \alphabet} \decdistunormed(\word \mid \words)
    \end{align}
\end{defin}
\begin{defin}\label{defn:global_constant}
    The \defn{global normalisation constant} (denoted $\constglobalnorm \in [0,1]$) and the \defn{context-specific global normalisation constant} (denoted $\constglobalnormcontext: \alphabet^{*} \to [0,1]$) are defined as:
    \begin{align}
        \constglobalnorm \defeq \sum_{\words' \in \alphabet^*} \prod_{t=1}^{|\words'|+1} \decdistunormed(\word'_t \mid \words'_{<t})
        ,\qquad\qquad
        \constglobalnormcontext(\words) \defeq \sum_{\words' \in \alphabet^*} \prod_{t=1}^{|\words'|+1} \decdistunormed(\word'_t \mid \words \circ \words'_{<t})
    \end{align}
\end{defin}
We note that $\constglobalnormcontext(\words)$ can be written recursively as:
\begin{subequations}
\begin{align}
    \constglobalnormcontext(\words) 
    &= \sum_{\words' \in \alphabet^*} \prod_{t=1}^{|\words'|+1} \decdistunormed(\word'_t \mid \words \circ \words'_{<t}) \\
    &= \decdistunormed(\eos \mid \words) + \sum_{\word \in \alphabet} \decdistunormed(\word \mid \words) \sum_{\words' \in \alphabet^*} \prod_{t=1}^{|\words'|+1} \decdistunormed(\word'_t \mid \words \circ \word \circ \words'_{<t}) \\
    &= \decdistunormed(\eos \mid \words) + \sum_{\word \in \alphabet} \decdistunormed(\word \mid \words)\,\constglobalnormcontext(\words \circ \word)
\end{align}
\end{subequations}
and that $\constglobalnorm = \constglobalnormcontext(\emptystring)$, where $\emptystring$ denotes an empty string.

Given these constants, we can rewrite local and global decoding algorithms, defined in \cref{eq:local_decoding,eq:global_decoding}, respectively, as:
\begin{align}
    &\decdist(\words) 
    = \prod_{t=1}^{|\words|+1} \frac{\decdistunormed(\word \mid \words_{<t})}{\sum_{\word' \in \alphabet} \decdistunormed(\word' \mid \words_{<t})} 
    = \frac{\prod_{t=1}^{|\words|+1} \decdistunormed(\word \mid \words_{<t})}{\prod_{t=1}^{|\words|+1} \constlocalnorm(\words_{<t})} \\
    &\decdistglobal(\words) 
    = \frac{\prod_{t=1}^{|\words|+1} \decdistunormed(\word_t \mid \words_{<t})}{\sum_{\words' \in \alphabet^*} \prod_{t=1}^{|\words|+1} \decdistunormed(\word_t \mid \words_{<t})} 
    = \frac{\prod_{t=1}^{|\words|+1} \decdistunormed(\word_t \mid \words_{<t})}{\constglobalnorm} 
\end{align}

We now prove a recursively-defined lower bound on $\constglobalnorm$ which will be useful later.
\begin{lemma} \label{lemma:bound_global_constant}
    The global normalisation constant's value is lower bounded by:
    \begin{align}
        \constglobalnorm \geq \left(\min_{\words \in \alphabet^{*}}\,\constlocalnorm(\words)\right)^{T} \left(\min_{\words \in \alphabet^{T}}\constglobalnormcontext(\words)\right)
    \end{align}
\end{lemma}
\begin{proof}
We start this proof by lower bounding the context-specific global normalisation constant with a local normalisation constant's value:
\begin{subequations}
\begin{align}
    \constglobalnormcontext(\words) 
    &= \decdistunormed(\eos \mid \words) + \sum_{\word \in \alphabet} \decdistunormed(\word \mid \words)\,\constglobalnormcontext(\words \circ \word) \\
    &\geq \decdistunormed(\eos \mid \words) + \sum_{\word \in \alphabet} \decdistunormed(\word \mid \words)\,\left(\min_{\word \in \alphabet}\, \constglobalnormcontext(\words \circ \word)\right) \\
    &\geq \left(\decdistunormed(\eos \mid \words) + \sum_{\word \in \alphabet} \decdistunormed(\word \mid \words)\right)\,\left(\min_{\word \in \alphabet}\, \constglobalnormcontext(\words \circ \word)\right) \\
    &= \constlocalnorm(\words)\,\min_{\word \in \alphabet}\, \constglobalnormcontext(\words \circ \word)
\end{align}
\end{subequations}
We can now recursively expand this lower bound:
\begin{subequations}
\begin{align}
    \constglobalnormcontext(\words) 
    &\geq \constlocalnorm(\words)\,\min_{\word \in \alphabet}\, \constglobalnormcontext(\words \circ \word) \\
    &\geq \constlocalnorm(\words)\,\min_{\word_1 \in \alphabet}\,\left(\constlocalnorm(\words\circ\word_1) \,\min_{\word_2 \in \alphabet}\constglobalnormcontext(\words \circ \word_1 \circ \word_2)\right) \\
    &\geq \constlocalnorm(\words)\,\min_{\word_1 \in \alphabet}\,\left(\constlocalnorm(\words\circ\word_1) \,\min_{\word_2 \in \alphabet}\,\left(\constlocalnorm(\words\circ\word_1 \circ \word_2) \,\min_{\word_3 \in \alphabet}\constglobalnormcontext(\words \circ \word_1 \circ \word_2 \circ \word_3)\right)\right) \\
    &\geq \constlocalnorm(\words)\,
    \min_{\word_1 \in \alphabet}\,\bigg(\constlocalnorm(\words\circ\word_1) \,
    \cdots \\
    &\qquad\qquad \min_{\word_{T-1} \in \alphabet}\,\bigg(\constlocalnorm(\words\circ\word_1 \circ \cdots \circ \word_{T-1}) \,
    \min_{\word_T \in \alphabet}\constglobalnormcontext(\words \circ \word_1 \circ \cdots \circ \word_{T-1}\circ \word_{T})\bigg)\bigg) \nonumber \\
    &= 
    \min_{\word_1 \in \alphabet}\,\bigg(
    \cdots
    \min_{\word_{T-1} \in \alphabet}\,\bigg( \\
    &\qquad\quad \constlocalnorm(\words)\,\constlocalnorm(\words\circ\word_1) \,\cdots \constlocalnorm(\words\circ\word_1 \circ \cdots \circ \word_{T-1}) \,
    \min_{\word_T \in \alphabet}\constglobalnormcontext(\words \circ \word_1 \circ \cdots \circ \word_{T-1}\circ \word_{T})\bigg)\bigg) \nonumber 
    \!\!\!\!\!\!\!\!\!\!\!\!\!\!\!\!\!\!\!\!\!\! \\
    &= \min_{\words' \in \alphabet^{T-1}}\,\left(\prod_{t=0}^{T-1} \constlocalnorm(\words\circ\words'_{\leq t}) \,\min_{\word_T \in \alphabet}\constglobalnormcontext(\words \circ \words' \circ \word_T)\right) \\
\end{align}
\end{subequations}
Finally, we get the bound above by noting that 
$\min_{\word'' \in \alphabet} \constglobalnormcontext(\words \circ \words' \circ \word'') \leq \min_{\words'' \in \alphabet^{T}}\constglobalnormcontext(\words \circ \words'')$
and
$\constlocalnorm(\words\circ\words'_{\leq t}) \leq \min_{\words'' \in \alphabet^{*}}\,\constlocalnorm(\words\circ\words'')$:
\begin{subequations}
\begin{align}
    \constglobalnormcontext(\words) 
    &\geq \min_{\words' \in \alphabet^{T-1}}\,\left(\prod_{t=0}^{T-1} \constlocalnorm(\words\circ\words'_{\leq t}) \,\min_{\word_T \in \alphabet}\constglobalnormcontext(\words \circ \words' \circ \word_T)\right) \\
    &\geq \min_{\words' \in \alphabet^{T-1}}\,\left(\prod_{t=0}^{T-1} \min_{\words'' \in \alphabet^{*}}\,\constlocalnorm(\words\circ\words'') \right) \left(\min_{\words'' \in \alphabet^{T}}\constglobalnormcontext(\words \circ \words'')\right) \\
    &= \left(\min_{\words'' \in \alphabet^{*}}\,\constlocalnorm(\words\circ\words'')\right)^{T} \left(\min_{\words'' \in \alphabet^{T}}\constglobalnormcontext(\words \circ \words'')\right)
\end{align}
\end{subequations}
Replacing $\constglobalnormcontext(\words)$ with $\constglobalnorm = \constglobalnormcontext(\emptystring)$ completes the proof.
\end{proof}

\section{Proof of Lower-bound on Maximum Divergence between Global and Local Distributions (\Cref{lemma:kl_lower_bound})}
\label{app:proofkllowerbound}

\kllowerbound*
\begin{proof}
    This proof follows trivially from \Cref{lemma:kl_lower_bound_forward,lemma:kl_lower_bound_reverse} below.
\end{proof}

\myword{\wordone}{a}
\myword{\wordsone}{\mathbf{a}}
\myword{\wordtwo}{b}
\myword{\wordthree}{c_{1}}
\myword{\wordfour}{c_2}
\myword{\wordlast}{c_{|\alphabet|-2}}

\begin{restatable}{lemma}{kllowerboundreverse}
\label{lemma:kl_lower_bound_reverse}
    Let $\variationalfamily$ be a set including all $\maxlength$-maxlength language models $\ptheta(\words)$ (see \cref{defn:maxlength_lang_model}).
    There exist language models $\ptheta \in \variationalfamily$, whose decoding versions $\decdistglobal(\words)$ and $\decdist(\words)$ have a reverse $\kl$ bounded below by:
    \begin{align}
        \sup_{\ptheta \in \variationalfamily}
        \kl(\decdist(\words) \mid\mid \decdistglobal(\words))
        \in \Omega(\maxlength)
    \end{align}
\end{restatable}
\begin{proof}
We prove this by construction. Let $\alphabet = \{\wordone, \wordtwo, \wordthree, ..., \wordlast\}$ and $\ptheta(\words)$ be defined such that:
\begin{align}
    \ptheta(\words) = \left\{ \begin{array}{cr}
         x & \words = \wordone \\
         (1-x)\,\frac{1}{|\alphabet|}^{\maxlength-1} & \words \in \wordtwo \circ \alphabet^{\maxlength-1} \\
         0 & \texttt{else}
    \end{array}\right.
\end{align}
where, when applied to sets, $\circ$ represents elementwise concatenation, i.e., $\wordtwo \circ \alphabet^{\maxlength-1} = \{\wordtwo \circ \words' \mid \words' \in  \alphabet^{\maxlength-1}\}$.
We can get a lower bound for this LM's reverse $\kl$ as: 
\begin{subequations}
\begin{align}
    \kl(\decdist(\words) \mathop{\mid\mid} \decdistglobal(\words)) 
    &= \sum_{\words \in \alphabet^*} \decdist(\words) \log \frac{\decdist(\words)}{\decdistglobal(\words)} \\
    & =
    \expect_{\words \sim \decdist(\words)} \left[
    \log \frac{\constglobalnorm}{\prod_{t=1}^{|\words|+1} \constlocalnorm(\words_{<t})}
    \right] 
    \\
    & =
    \log \constglobalnorm + \expect_{\words \sim \decdist(\words)} \left[
    \log \frac{1}{\prod_{t=1}^{|\words|+1} \constlocalnorm(\words_{<t})}
    \right]
    \\
    & =
    \log \constglobalnorm + \sum_{\words \in \wordtwo \circ \alphabet^{\maxlength-1}} \decdist(\words)\, 
    \log \frac{1}{\prod_{t=1}^{|\words|+1} \constlocalnorm(\words_{<t})}
    \\
    & =
    \log \underbrace{\left(x + \underbrace{\sum_{\words \in \wordtwo \circ \alphabet^{\maxlength-1}} (1-x)\,\frac{1}{|\alphabet|}^{\maxlength-1}}_{\geq 0} \right)}_{\constglobalnorm} + 
    (1-x)\, \log \frac{1}{\prod_{t=1}^{\maxlength+1} \constlocalnorm(\words_{<t})}
    \\
    & \geq
    \log x + 
    (1-x)\, \log \frac{1}{\underbrace{\constlocalnorm(\emptystring)}_{=1} \cdot \prod_{t=2}^{\maxlength} \constlocalnorm(\words_{<t}) \cdot \underbrace{\constlocalnorm(\words)}_{=1}}
    \\
    & =
    \log x + 
    (1-x)\, \log \frac{1}{\prod_{t=2}^{\maxlength} \constlocalnorm(\words_{<t})}
    \\
    & =
    \log x + 
    (1-x)\, \sum_{t=2}^{\maxlength} \log \frac{1}{\constlocalnorm(\words_{<t})}
    \\
    & =
    \log x + 
    (1-x)\, (\maxlength-1) \log \frac{1}{\constlocalnorm(\words_{<t})} 
    \in \Omega(\maxlength)
\end{align}
\end{subequations}
The $\sup_{\ptheta \in \variationalfamily} \kl(\decdist(\words) \mid\mid \decdistglobal(\words))$ is greater or equal to this LM's $\kl$, and so is also bounded below. 
This completes the proof.
\end{proof}

\begin{restatable}{lemma}{kllowerboundforward}
\label{lemma:kl_lower_bound_forward}
    Let $\variationalfamily$ be a set including all $\maxlength$-maxlength language models $\ptheta(\words)$ (see \cref{defn:maxlength_lang_model}).
    There exist language models $\ptheta \in \variationalfamily$, whose decoding versions $\decdistglobal(\words)$ and $\decdist(\words)$ have a forward $\kl$ bounded below by:
    \begin{align}
        \sup_{\ptheta \in \variationalfamily}
        \kl(\decdistglobal(\words) \mid\mid \decdist(\words))
        \in \Omega(\maxlength)
    \end{align}
\end{restatable}
\begin{proof}
\allowdisplaybreaks
We now prove this by construction for top-$\topkk$, but note that a similar proof applies for top-$\toppp$.
Let $\alphabet = \{\wordone, \wordtwo, \wordthree, ..., \wordlast\}$ and $\ptheta(\words)$ be defined such that:
\begin{align}
    \ptheta(\words) = \left\{ \begin{array}{cr}
         x^T & \words = \wordone_1\circ \wordone_2 \circ \cdots \circ \wordone_T \\
         x^{t}\,(1-x)\,\frac{1}{|\alphabet|}^{\maxlength-t-1} & \words \in \wordone_1 \circ \cdots \circ \wordone_t \circ \wordtwo_{t+1} \circ \alphabet^{\maxlength-t-1} \\
         0 & \texttt{else}
    \end{array}\right.
\end{align}
Further, let $1 > x > \frac{\topkk}{|\alphabet|}$.
First, we simplify the value of the global normalisation constant for this model:
\begin{subequations}
\begin{align}
    \constglobalnorm 
    &= \sum_{\words' \in \alphabet^*} \prod_{t=1}^{|\words'|+1} \decdistunormed(\word'_t \mid \words'_{<t}) \\
    &= \prod_{t=1}^{|\wordsone_{1:\maxlength}|+1} \decdistunormed(\wordone \mid \wordsone_{1:t-1}) \\
    &\qquad\qquad + \sum_{\words \in \wordsone_{1:i}\circ\wordtwo \circ \alphabet^{\maxlength-i-1}} 
    \left(\left(\prod_{t=1}^{i} \decdistunormed(\wordone_t \mid \words_{<t})\right)
    \decdistunormed(\wordtwo_t \mid \words_{\leq i})
    \left(\prod_{t=i+1}^{\maxlength+1} \decdistunormed(\word_t \mid \words_{<t}) \right)\right) \nonumber \\
    &= x^{\maxlength} + \sum_{\words \in \wordsone_{1:i}\circ\wordtwo \circ \alphabet^{\maxlength-i-1}} x^{i}\,(1-x)\,\frac{1}{|\alphabet|}^{\maxlength-i-1}\,\prod_{t=1}^{\maxlength+1} \one\{\word_t \in \alphabetkeep(\words_{<t})\} \\
    &= x^{\maxlength} + \sum_{i=0}^{\maxlength-1} x^{i}\,(1-x)\,\frac{1}{|\alphabet|}^{\maxlength-i-1}\,\topkk^{\maxlength-i-1} \\
    &= x^{\maxlength} + (1-x)\,\frac{\topkk}{|\alphabet|}^{\maxlength-1}\,\sum_{i=0}^{\maxlength-1} (\frac{x\,|\alphabet|}{\topkk})^{i} \\
    &= x^{\maxlength} + (1-x)\,\frac{\topkk}{|\alphabet|}^{\maxlength-1}\,\frac{1-(\frac{x\,|\alphabet|}{\topkk})^{\maxlength}}{1-\frac{x\,|\alphabet|}{\topkk}}
\end{align}
\end{subequations}

Now, we simplify the value of $\expect_{\words \sim \decdistglobal(\words)} \left[\log \prod_{t=1}^{|\words|+1} \constlocalnorm(\words_{<t})\right]$:
\begin{subequations}
\begin{align}
    \expect_{\words \sim \decdistglobal(\words)} &\left[\log \prod_{t=1}^{|\words|+1} \constlocalnorm(\words_{<t})\right] \nonumber \\
    & =
    \sum_{\words \in \alphabet^{*}} \decdistglobal(\words)\, 
    \log \prod_{t=1}^{|\words|+1} \constlocalnorm(\words_{<t}) \\
    & =
    \sum_{\words \in \wordsone_{1:i}\circ\wordtwo \circ \alphabet^{\maxlength-i-1}} \decdistglobal(\words)\, 
    \log \prod_{t=1}^{|\words|+1} \constlocalnorm(\words_{<t}) \\
    & =
    \sum_{\words \in \wordsone_{1:i}\circ\wordtwo \circ \alphabet^{\maxlength-i-1}} 
    \frac{x^{i}\,(1-x)\,\frac{1}{|\alphabet|}^{\maxlength-t-1}}{
    \constglobalnorm}
    \log \prod_{t=1}^{|\words|+1} \constlocalnorm(\words_{<t}) \\
    & =
    \frac{\sum_{\words \in \wordsone_{1:i}\circ\wordtwo \circ \alphabet^{\maxlength-i-1}} x^{i}\,(1-x)\,\frac{1}{|\alphabet|}^{\maxlength-t-1}
    \log \prod_{t=1}^{|\words|+1} \constlocalnorm(\words_{<t})}{
    \constglobalnorm} \\
    & =
    \frac{(1-x)\,\frac{\topkk}{|\alphabet|}^{\maxlength-1}\,\sum_{i = 0}^{\maxlength-1} (\frac{x\,|\alphabet|}{\topkk})^{i}
    \log \prod_{t=1}^{|\words|+1} \constlocalnorm(\words_{<t})}{
    \constglobalnorm} \\
    & =
    \frac{(1-x)\,\frac{\topkk}{|\alphabet|}^{\maxlength-1}\,\sum_{i = 0}^{\maxlength-1} (\frac{x\,|\alphabet|}{\topkk})^{i}
    \log \left(\underbrace{\prod_{t=1}^{i} \constlocalnorm(\words_{<t})}_{=1}\,\prod_{t=i+1}^{\maxlength} \constlocalnorm(\words_{<t})\,\underbrace{\constlocalnorm(\words_{<\maxlength+1})}_{=1}\right)}{
    \constglobalnorm} \\
    & =
    \frac{(1-x)\,\frac{\topkk}{|\alphabet|}^{\maxlength-1}\,\sum_{i = 0}^{\maxlength-1} (\frac{x\,|\alphabet|}{\topkk})^{i}
    \log \prod_{t=i+1}^{\maxlength} \frac{\topkk}{|\alphabet|}}{
    \constglobalnorm} \\
    & =
    \frac{(1-x)\,\frac{\topkk}{|\alphabet|}^{\maxlength-1}\,\sum_{i = 0}^{\maxlength-1} (\frac{x\,|\alphabet|}{\topkk})^{i}
    (\maxlength - i)\,\log \frac{\topkk}{|\alphabet|}}{
    \constglobalnorm} \\
    & =
    \frac{(1-x)\,\frac{\topkk}{|\alphabet|}^{\maxlength-1}
    \left(
    \sum_{i = 0}^{\maxlength-1} (\frac{x\,|\alphabet|}{\topkk})^{i}
    (\maxlength + 1) -
    \sum_{i = 0}^{\maxlength-1} (\frac{x\,|\alphabet|}{\topkk})^{i}
    (i + 1)
    \right)
    \log \frac{\topkk}{|\alphabet|}}{
    \constglobalnorm} \\
    & =
    \frac{(1-x)\,\frac{\topkk}{|\alphabet|}^{\maxlength-1}
    \left(
    (\maxlength+1)\frac{1-(\frac{x\,|\alphabet|}{\topkk})^{\maxlength}}{1-\frac{x\,|\alphabet|}{\topkk}} -
    \frac{\maxlength(\frac{x\,|\alphabet|}{\topkk})^{\maxlength+1} - (\maxlength+1)(\frac{x\,|\alphabet|}{\topkk})^{\maxlength}+1}{(1-\frac{x\,|\alphabet|}{\topkk})^2}
    \right)
    \log \frac{\topkk}{|\alphabet|}}{
    \constglobalnorm} \\
    & =
    \frac{(1-x)\,\frac{\topkk}{|\alphabet|}^{\maxlength-1}
    \left(
    (\maxlength+1)\frac{1-\frac{x\,|\alphabet|}{\topkk}-(\frac{x\,|\alphabet|}{\topkk})^{\maxlength}+(\frac{x\,|\alphabet|}{\topkk})^{\maxlength+1}}{(1-\frac{x\,|\alphabet|}{\topkk})^2} -
    \frac{\maxlength(\frac{x\,|\alphabet|}{\topkk})^{\maxlength+1} - (\maxlength+1)(\frac{x\,|\alphabet|}{\topkk})^{\maxlength}+1}{(1-\frac{x\,|\alphabet|}{\topkk})^2}
    \right)
    \log \frac{\topkk}{|\alphabet|}}{
    \constglobalnorm} \\
    & =
    \frac{(1-x)\,\frac{\topkk}{|\alphabet|}^{\maxlength-1}
    \left(
    \frac{
    (\frac{x\,|\alphabet|}{\topkk})^{\maxlength+1} - (\maxlength+1)\frac{x\,|\alphabet|}{\topkk} + \maxlength}
    {(1-\frac{x\,|\alphabet|}{\topkk})^2}
    \right)
    \log \frac{\topkk}{|\alphabet|}}{
    \constglobalnorm}
\end{align}
\end{subequations}

Note that the $\kl$ we are interested in is defined as:
\begin{align}
    \kl(\decdistglobal(\words) \mathop{\mid\mid} \decdist(\words)) 
    & =
    \expect_{\words \sim \decdistglobal(\words)} \left[
    \log \frac{\prod_{t=1}^{|\words|+1} \constlocalnorm(\words_{<t})}{\constglobalnorm}
    \right]
\end{align}
so we can fill in the values above into it.
Now we prove this $\kl \in \Omega(\maxlength)$. 
To do so, we show that the $\lim_{\maxlength\to\infty} \frac{\kl}{\maxlength} = C$, for a $C > 0$.
First, we isolate the terms dependent on $\maxlength$ in the $\kl$'s equation.
\begin{subequations}
\begin{align}
    &\kl(\decdistglobal(\words) \mathop{\mid\mid} \decdist(\words))
    =
    \log \frac{1}{\constglobalnorm} +
    \expect_{\words \sim \decdistglobal(\words)} \left[\log \prod_{t=1}^{|\words|+1} \constlocalnorm(\words_{<t}) \right] \\
    & =
    \log \frac{1}{x^{\maxlength} + (1-x)\,\frac{\topkk}{|\alphabet|}^{\maxlength-1}\,\frac{1-(\frac{x\,|\alphabet|}{\topkk})^{\maxlength}}{1-\frac{x\,|\alphabet|}{\topkk}}} +
    \frac{(1-x)\,\frac{\topkk}{|\alphabet|}^{\maxlength-1}
    \left(
    \frac{
    (\frac{x\,|\alphabet|}{\topkk})^{\maxlength+1} - (\maxlength+1)\frac{x\,|\alphabet|}{\topkk} + \maxlength}
    {(1-\frac{x\,|\alphabet|}{\topkk})^2}
    \right)
    \log \frac{\topkk}{|\alphabet|}}{
    x^{\maxlength} + (1-x)\,\frac{\topkk}{|\alphabet|}^{\maxlength-1}\,\frac{1-(\frac{x\,|\alphabet|}{\topkk})^{\maxlength}}{1-\frac{x\,|\alphabet|}{\topkk}}} \\
    & =
    \log \frac{x^{-\maxlength}}{1 + \frac{1-x}{\frac{\topkk}{|\alphabet|}-x}\,\left((\frac{\topkk}{|\alphabet|}\frac{1}{x})^{\maxlength}-1\right)
    } +
    \frac{\frac{1-x}{\frac{\topkk}{|\alphabet|}}\,\frac{1}{(1-\frac{x\,|\alphabet|}{\topkk})^2}
    \left(
    (\frac{\topkk}{|\alphabet|}\frac{1}{x})^{-1}
    \!\!-\! (\maxlength+1) (\frac{\topkk}{|\alphabet|}\frac{1}{x})^{\maxlength-1}
    \!\!+\! \maxlength\,(\frac{\topkk}{|\alphabet|}\frac{1}{x})^{\maxlength}
    \right)
    \log \frac{\topkk}{|\alphabet|}}
    {1 + \frac{1-x}{\frac{\topkk}{|\alphabet|}-x}\,\left((\frac{\topkk}{|\alphabet|}\frac{1}{x})^{\maxlength}-1\right)} \\
    & =
    \log \frac{x^{-\maxlength}}{1 + \frac{1-x}{\frac{\topkk}{|\alphabet|}-x}\,\left((\frac{\topkk}{|\alphabet|}\frac{1}{x})^{\maxlength}-1\right)
    } +
    \frac{\frac{1-x}{\frac{\topkk}{|\alphabet|}\,\frac{\topkk}{|\alphabet|}\frac{1}{x}}\,\frac{1}{(1-\frac{x\,|\alphabet|}{\topkk})^2}
    \left(
    1
    - (\maxlength+1) (\frac{\topkk}{|\alphabet|}\frac{1}{x})^{\maxlength}
    + \maxlength\,(\frac{\topkk}{|\alphabet|}\frac{1}{x})^{\maxlength+1}
    \right)
    \log \frac{\topkk}{|\alphabet|}}
    {1 + \frac{1-x}{\frac{\topkk}{|\alphabet|}-x}\,\left((\frac{\topkk}{|\alphabet|}\frac{1}{x})^{\maxlength}-1\right)} \\
    & =
    \log \frac{x^{-\maxlength}}{1 + \frac{1-x}{\frac{\topkk}{|\alphabet|}-x}\,\left((\frac{\topkk}{|\alphabet|}\frac{1}{x})^{\maxlength}-1\right)
    } +
    \frac{\frac{(1-x)\,x}{\left(\frac{\topkk}{|\alphabet|}-x\right)^2}
    \left(
    1
    - (\maxlength+1) (\frac{\topkk}{|\alphabet|}\frac{1}{x})^{\maxlength}
    + \maxlength\,(\frac{\topkk}{|\alphabet|}\frac{1}{x})^{\maxlength+1}
    \right)
    \log \frac{\topkk}{|\alphabet|}}
    {1 + \frac{1-x}{\frac{\topkk}{|\alphabet|}-x}\,\left((\frac{\topkk}{|\alphabet|}\frac{1}{x})^{\maxlength}-1\right)} \\
    &=
    -\maxlength\,\log x - \log \left(1 + \frac{1-x}{\frac{\topkk}{|\alphabet|}-x}\,\left((\frac{\topkk}{|\alphabet|}\frac{1}{x})^{\maxlength}-1\right)\right) + \\
    &\qquad\qquad\qquad\qquad\qquad\qquad \frac{\frac{(1-x)\,x}{\left(\frac{\topkk}{|\alphabet|}-x\right)^2}
    \left(
    1
    - (\maxlength+1) (\frac{\topkk}{|\alphabet|}\frac{1}{x})^{\maxlength}
    + \maxlength\,(\frac{\topkk}{|\alphabet|}\frac{1}{x})^{\maxlength+1}
    \right)
    \log \frac{\topkk}{|\alphabet|}}
    {1 + \frac{1-x}{\frac{\topkk}{|\alphabet|}-x}\,\left((\frac{\topkk}{|\alphabet|}\frac{1}{x})^{\maxlength}-1\right)} \nonumber
\end{align}
\end{subequations}
We now analyse the limit:
$\lim_{\maxlength \to \infty} \frac{\kl(\decdistglobal(\words) \mathop{\mid\mid} \decdist(\words))}{\maxlength}$.
Note that, by construction, $1 > x > \frac{\topkk}{|\alphabet|}$.
We thus write:
\begin{subequations}
\begin{align}
    &\lim_{\maxlength \to \infty} 
    \frac{\kl(\decdistglobal(\words) \mathop{\mid\mid} \decdist(\words))}{\maxlength} \nonumber \\
    &=
    \lim_{\maxlength \to \infty} 
    \frac{-\maxlength\,\log x - \log \left(1 + \frac{1-x}{\frac{\topkk}{|\alphabet|}-x}\,\left((\frac{\topkk}{|\alphabet|}\frac{1}{x})^{\maxlength}-1\right)\right) +
    \frac{\frac{(1-x)\,x}{\left(\frac{\topkk}{|\alphabet|}-x\right)^2}
    \left(
    1
    - (\maxlength+1) (\frac{\topkk}{|\alphabet|}\frac{1}{x})^{\maxlength}
    + \maxlength\,(\frac{\topkk}{|\alphabet|}\frac{1}{x})^{\maxlength+1}
    \right)
    \log \frac{\topkk}{|\alphabet|}}
    {1 + \frac{1-x}{\frac{\topkk}{|\alphabet|}-x}\,\left((\frac{\topkk}{|\alphabet|}\frac{1}{x})^{\maxlength}-1\right)}}
    {\maxlength} \\
    &=
    \lim_{\maxlength \to \infty} 
    \frac{-\maxlength\,\log x}{\maxlength} +
    \frac{\frac{(1-x)\,x}{\left(\frac{\topkk}{|\alphabet|}-x\right)^2}
    \left(
    1
    - (\maxlength+1) (\frac{\topkk}{|\alphabet|}\frac{1}{x})^{\maxlength}
    + \maxlength\,(\frac{\topkk}{|\alphabet|}\frac{1}{x})^{\maxlength+1}
    \right)
    \log \frac{\topkk}{|\alphabet|}}
    {\maxlength\left(1 + \frac{1-x}{\frac{\topkk}{|\alphabet|}-x}\,\left((\frac{\topkk}{|\alphabet|}\frac{1}{x})^{\maxlength}-1\right)\right)} \\
    &= - \log x > 0
\end{align}
\end{subequations}
This completes the proof.
\end{proof}

\section{Upper-bounding the Divergence between Global and Local Distributions (\Cref{lemma:kl_upper_bound_both})}
\label{app:proofklupperbound_both}

\klupperboundboth*
\begin{proof}
    This proof follows trivially from \Cref{lemma:kl_upper_bound_topk,lemma:kl_upper_bound_topp} below.
\end{proof}

\subsection{A General Upper-bound}

In this section, we prove an upper-bound on the KL between both decoding distributions.
We then provide corollaries discussing how this bound is instantiated by top-$\topkk$ and top-$\toppp$ algorithms.

\begin{lemma} \label{lemma:kl_upper_bound}
    Let $\ptheta(\words)$ be a $\maxlength$-maxlength language model (see \cref{defn:maxlength_lang_model}).
    Now let $\decdistglobal(\words)$ and $\decdist(\words)$ be global and local decoding algorithms run on top of $\ptheta(\words)$.
    In this case, the forward $\kl$ between $\decdistglobal(\words)$ and $\decdist(\words)$ is bounded above by:
    \begin{align}
        \kl(\decdistglobal(\words) \mid\mid \decdist(\words)) \leq 
        T\,\log \frac{1}{\min_{\words \in \alphabet^{*}}\,\constlocalnorm(\words)}
    \end{align}
\end{lemma}
\begin{proof}
First, we use the definition of the KL to show a bound:
\begin{subequations}
\begin{align}
    \kl(\decdistglobal(\words) \mathop{\mid\mid} \decdist(\words)) 
    &= \sum_{\words \in \alphabet^*} \decdistglobal(\words) \log \frac{\decdistglobal(\words)}{\decdist(\words)} \\
    & =
    \expect_{\words \sim \decdistglobal(\words)} \left[
    \log \frac{\frac{\prod_{t=1}^{|\words|+1} \decdistunormed(\word_t \mid \words_{<t})}{\constglobalnorm}}
    {\frac{\prod_{t=1}^{|\words|+1} \decdistunormed(\word \mid \words_{<t})}{\prod_{t=1}^{|\words|+1} \constlocalnorm(\words_{<t})}}
    \right] 
    & 
    \!\!\!\!\!\!\!\!\!\!\!\!\!\!\!\!\!\!\!\!\!\!
    \mathcomment{definition of $\decdistglobal(\words)$ and $\decdist(\words)$} \\
    & =
    \expect_{\words \sim \decdistglobal(\words)} \left[
    \log \frac{\prod_{t=1}^{|\words|+1} \constlocalnorm(\words_{<t})}{\constglobalnorm}
    \right] 
    & 
    \!\!\!\!\!\!\!\!\!\!\!\!\!\!\!\!\!\!\!\!\!\!\!\!\!\!\!\! 
    \mathcomment{cancel terms} 
    \\
    & =
    \expect_{\words \sim \decdistglobal(\words)} \left[
    \log \prod_{t=1}^{|\words|+1} \constlocalnorm(\words_{<t})\right]
    + \log \frac{1}{\constglobalnorm} 
    & 
    \!\!\!\!\!\!\!\!\!\!\!\!\!\!\!\!\!\!\!\!\!\!\!\!\!\!\!\! 
    \mathcomment{$\constglobalnorm$ does not depend on $\words$} \\
    & \leq \log \frac{1}{\constglobalnorm} 
    & 
    \!\!\!\!\!\!\!\!\!\!\!\!\!\!\!\!\!\!\!\!\!\!\!\!\!\!\!\! 
    \mathcomment{first term is $\leq 0$} \\
    & \leq \log \frac{1}{\left(\min_{\words \in \alphabet^{*}}\,\constlocalnorm(\words)\right)^{T} \left(\min_{\words \in \alphabet^{T}}\constglobalnormcontext(\words)\right)} 
    & 
    \!\!\!\!\!\!\!\!\!\!\!\!\!\!\!\!\!\!\!\!\!\!\!\!\!\!\!\! 
    \mathcomment{apply \cref{lemma:bound_global_constant}}\\
    & = T\,\log \frac{1}{\min_{\words \in \alphabet^{*}}\,\constlocalnorm(\words)} 
    + \log \frac{1}{\min_{\words \in \alphabet^{T}}\constglobalnormcontext(\words)}
\end{align}
\end{subequations}
Now note that $\constglobalnormcontext(\words) = \decdistunormed(\eos \mid \words) + \sum_{\word \in \alphabet} \decdistunormed(\word \mid \words)\,\constglobalnormcontext(\words \circ \word)$.
Since our language model is $\maxlength$-maxlengthed, then $\ptheta(\eos \mid \words) = 1$, which implies that $\decdistunormed(\eos \mid \words) = 1$ for all $\words \in \alphabet^{T}$.
We thus have that $\min_{\words \in \alphabet^{T}}\constglobalnormcontext(\words) = 1$. 
This completes the proof.
\end{proof}

\begin{lemma} \label{lemma:kl_reverse_upper_bound}
    Let $\ptheta(\words)$ be a $\maxlength$-maxlength language model (see \cref{defn:maxlength_lang_model}).
    Now let $\decdistglobal(\words)$ and $\decdist(\words)$ be global and local decoding algorithms run on top of $\ptheta(\words)$.
    In this case, the reverse $\kl$ between $\decdistglobal(\words)$ and $\decdist(\words)$ is bounded above by:
    \begin{align}
        \kl(\decdist(\words) \mid\mid \decdistglobal(\words)) \leq 
        T\,\log \frac{1}{\min_{\words \in \alphabet^{*}}\,\constlocalnorm(\words)}
    \end{align}
\end{lemma}
\begin{proof}
For this proof, we start with the $\kl$'s definition and show the upper-bound:
\begin{subequations}
\begin{align}
    \kl(\decdist(\words) \mathop{\mid\mid} \decdistglobal(\words)) 
    &= \sum_{\words \in \alphabet^*} \decdist(\words) \log \frac{\decdist(\words)}{\decdistglobal(\words)} \\
    & =
    \expect_{\words \sim \decdist(\words)} \left[
    \log \frac{\constglobalnorm}{\prod_{t=1}^{|\words|+1} \constlocalnorm(\words_{<t})}
    \right] 
    \\
    & =
    \log \constglobalnorm + \expect_{\words \sim \decdist(\words)} \left[
    \log \frac{1}{\prod_{t=1}^{|\words|+1} \constlocalnorm(\words_{<t})}
    \right]
    \\
    & \leq
    \expect_{\words \sim \decdist(\words)} \left[
    \log \frac{1}{\prod_{t=1}^{|\words|+1} \constlocalnorm(\words_{<t})}
    \right]
    \\
    & \leq
    \expect_{\words \sim \decdist(\words)} \left[
    \log \frac{1}{\prod_{t=1}^{\maxlength} \min_{\words' \in \alphabet^{*}}\,\constlocalnorm(\words')}
    \right]
    \\
    & =
    \log \frac{1}{\prod_{t=1}^{\maxlength} \min_{\words \in \alphabet^{*}}\,\constlocalnorm(\words)}
    \\
    & =
    \maxlength\, \log \frac{1}{\min_{\words \in \alphabet^{*}}\,\constlocalnorm(\words)}
\end{align}
\end{subequations}
This concludes the proof.
\end{proof}

\subsection{An Upper-bound for Top-$\topkk$ (\Cref{lemma:kl_upper_bound_topk})}
\label{app:proofklupperbound_topk}

\begin{restatable}{lemma}{klupperboundtopk}
\label{lemma:kl_upper_bound_topk}
    When using top-$\topkk$ decoding, both $\kl$s (forward and reverse) between $\decdistglobal(\words)$ and $\decdist(\words)$ are bounded above by:
    \begin{subequations}
    \begin{align}
        \kl(\decdistglobal(\words) \mid\mid \decdist(\words)) \leq 
        \maxlength\,\log \frac{|\alphabeteos|}{\topkk}, \\
        \kl(\decdist(\words) \mid\mid \decdistglobal(\words)) \leq 
        \maxlength\,\log \frac{|\alphabeteos|}{\topkk}
    \end{align}
    \end{subequations}
    where $\ptheta(\words)$ is a $\maxlength$-maxlength language model.
\end{restatable}
\begin{proof}
    For convenience, we first rewrite the definition of the set of strings unpruned by top-$k$ here:
    \begin{align}
        \alphabetkeep(\words_{<t}) = &\argmax_{\alphabetkeep' \subseteq \alphabeteos} \sum_{\word \in \alphabetkeep'} \ptheta(\word \mid \words_{<t}),\,\, 
        \mathrm{s.t.}\, |\alphabetkeep'| = k
    \end{align}
    Top-$k$'s $\alphabetkeep(\words_{<t})$ is thus defined as the largest probability $k$-sized subset of $\alphabeteos$.
    This set clearly has at least probability $\frac{k}{|\alphabeteos|}$.
    We can thus bound the local constant's value as:
    \begin{align}
        \constlocalnorm(\words) 
        = \sum_{\word \in \alphabeteos} \decdistunormed(\word \mid \words) 
        = \sum_{\word \in \alphabeteos} \ptheta(\word \mid \words)\, \one\{\word \in \alphabetkeep(\words)\}  
        \geq \frac{k}{|\alphabeteos|}
    \end{align}
    We can now apply this inequality to the bound in \cref{lemma:kl_upper_bound}:
    \begin{subequations}
    \begin{align}
        \kl(\decdistglobal(\words) \mid\mid \decdist(\words)) 
        &\leq 
        \maxlength\,\log \frac{1}{\min_{\words \in \alphabet^{*}}\,\constlocalnorm(\words)} \\
        &\leq 
        \maxlength\,\log \frac{|\alphabeteos|}{k}
    \end{align}
    \end{subequations}
    The same logic applies to the $\kl(\decdist(\words) \mid\mid \decdistglobal(\words))$.
    This completes the proof.
\end{proof}

\subsection{An Upper-bound for Top-$\toppp$ (\Cref{lemma:kl_upper_bound_topp})}
\label{app:proofklupperbound_topp}

\begin{restatable}{lemma}{klupperboundtopp}
\label{lemma:kl_upper_bound_topp}
    Let $\ptheta(\words)$ be a $\maxlength$-maxlength language model (see \cref{defn:maxlength_lang_model}).
    Now let $\decdistglobal(\words)$ and $\decdist(\words)$ be global and local decoding algorithms based on top-$\toppp$ (as in \cref{defn:topp_decoding}).
    In this case, both $\kl$s (forward and reverse) between $\decdistglobal(\words)$ and $\decdist(\words)$ are bounded above by:
    \begin{subequations}
    \begin{align}
        \kl(\decdistglobal(\words) \mid\mid \decdist(\words)) \leq 
        \maxlength\,\log \frac{1}{\toppp}, \\
        \kl(\decdist(\words) \mid\mid \decdistglobal(\words)) \leq 
        \maxlength\,\log \frac{1}{\toppp}
    \end{align}
    \end{subequations}
\end{restatable}
\begin{proof}
    For convenience, we first rewrite the definition of the set of strings unpruned by top-$\toppp$ here:
    \begin{align}
        \alphabetkeep(\words_{<t}) = &\argmin_{\alphabetkeep' \subseteq \alphabeteos} \alphabetkeep',\quad \mathrm{s.t.}\sum_{\word \in \alphabetkeep'} \ptheta(\word \mid \words_{<t}) \geq \toppp
    \end{align}
    Top-$\toppp$'s $\alphabetkeep(\words_{<t})$ is thus defined as the smallest subset of $\alphabeteos$ which has at least probability $\toppp$.
    We can thus bound the local constant's value as:
    \begin{align}
        \constlocalnorm(\words) 
        = \sum_{\word \in \alphabeteos} \decdistunormed(\word \mid \words) 
        = \sum_{\word \in \alphabeteos} \ptheta(\word \mid \words)\, \one\{\word \in \alphabetkeep(\words)\}  
        \geq \toppp
    \end{align}
    We can now apply this inequality to the bound in \cref{lemma:kl_upper_bound}:
    \begin{subequations}
    \begin{align}
        \kl(\decdistglobal(\words) \mid\mid \decdist(\words)) 
        &\leq 
        \maxlength\,\log \frac{1}{\min_{\words \in \alphabet^{*}}\,\constlocalnorm(\words)} \\
        &\leq 
        \maxlength\,\log \frac{1}{\toppp}
    \end{align}
    \end{subequations}
    The same logic applies to the $\kl(\decdist(\words) \mid\mid \decdistglobal(\words))$.
    This completes the proof.
\end{proof}

\end{document}